\documentclass{article} %
\usepackage{iclr2027_conference,times}

\usepackage{amsmath,amsfonts,bm}

\def\eqref#1{equation~\ref{#1}}

\def\1{\bm{1}}

\DeclareMathAlphabet{\mathsfit}{\encodingdefault}{\sfdefault}{m}{sl}
\SetMathAlphabet{\mathsfit}{bold}{\encodingdefault}{\sfdefault}{bx}{n}

\input{insbox.tex}

\usepackage{hyperref}
\usepackage{url}
\usepackage{graphicx}
\usepackage{booktabs}
\usepackage{multirow}
\usepackage{array}
\usepackage{makecell}
\usepackage[table]{xcolor}
\usepackage{pifont}
\usepackage{subcaption}
\usepackage{algorithm}
\usepackage{algorithmic}
\usepackage{mathtools}
\usepackage{amsthm}
\usepackage{amssymb}
\usepackage{wrapfig}
\usepackage{tcolorbox}          
\usepackage{pifont}

\definecolor{lightgrey}{RGB}{242, 242, 242}
\definecolor{average}{RGB}{240, 204, 126}
\definecolor{last}{RGB}{225, 191, 192}
\definecolor{transfer}{RGB}{191, 191, 225}
\definecolor{diag}{RGB}{151, 205, 112}

\definecolor{darkblue}{RGB}{84, 112, 198}
\definecolor{lightgreen}{RGB}{145, 204, 117}
\definecolor{lightyellow}{RGB}{250, 200, 88}
\definecolor{lightred}{RGB}{238, 102, 102}
\definecolor{lightblue}{RGB}{115, 192, 222}

\definecolor{outergrey}{HTML}{CCCCCC}  %
\definecolor{questionbg}{HTML}{FFF8F0}   %
\definecolor{gspobg1}{HTML}{F0F8FF}    %
\definecolor{gspobg2}{HTML}{FFF2F2}    %
\definecolor{oursbg}{HTML}{F2FAF2}     %
\definecolor{titlebrown}{HTML}{8B4513}  %

\newtcolorbox{promptbox}[2][Prompt]{
colback=black!5!white,
arc=5pt, 
boxrule=0.5pt,
fonttitle=\bfseries,
title=#1, 
before upper={\small}, fontupper=\fontfamily{ptm}\selectfont,
colframe=#2,
}

\newtcolorbox{outerbox}{
    colback=white,
    colframe=outergrey,
    arc=0mm,          %
    boxrule=1.2pt,    %
    boxsep=0.5pt,       %
    left=0.5pt,         %
    right=0.5pt,        %
    top=0.5pt,          %
    bottom=0.5pt,       %
}

\newtcolorbox{innerblock}[1]{
    colback=#1,       %
    colframe=white,
    arc=0mm,
    boxrule=0pt,
    top=1pt, bottom=1pt, left=1pt, right=1pt,
    before skip=1pt, after skip=1pt,
}

\newtheorem{proposition}{Proposition}[section]

\title{RL Forgets! Towards Continual Policy \\ Optimization}

\author{
  Mao-Lin Luo\textmd{\textsuperscript{1,2}}\footnotemark[1]\textmd{,} ~Zhe-Xu Wang\textmd{\textsuperscript{1,2}}\footnotemark[1]\textmd{,} ~Zi-Hao Zhou\textmd{\textsuperscript{1,2}}\footnotemark[1]\textmd{,} ~Bo Ye\textmd{\textsuperscript{1,2,3}}\textmd{,} ~Jian Zhao\textmd{\textsuperscript{3,4}}\textmd{,} \\ ~\textbf{Min-Ling Zhang}\textmd{\textsuperscript{1,2}}\textmd{,} ~\textbf{Tong Wei}\textmd{\textsuperscript{1,2}}\footnotemark[2] \\
  \textsuperscript{1}School of Computer Science and Engineering, Southeast University, Nanjing 210096, China\\ \textsuperscript{2}Key Laboratory of Computer Network and Information Integration (Southeast University), \\
  ~~Ministry of Education, China\\ 
  \textsuperscript{3}Zhongguancun Academy \\
  \textsuperscript{4}Zhongguancun Institute of Artificial Intelligence \\
}

\iclrfinalcopy %
\begin{document}

\maketitle

\footnotetext[1]{Equal contribution.}
\footnotetext[2]{Corresponding author.}

\begin{abstract}

Continual post-training is becoming a central paradigm for adapting vision-language models to evolving tasks. Recent studies report that reinforcement learning is less prone to forgetting than supervised fine-tuning, motivating the view that RL is inherently resistant to forgetting. However, this view remains insufficiently validated, as existing evidence is largely drawn from outdated or homogeneous benchmarks. We revisit this assumption by introducing MRCL, a Multimodal Reasoning Continual Learning benchmark built from recent and diverse multimodal reasoning tasks. Experiments on MRCL show that standard reinforcement learning still suffers from  catastrophic forgetting during continual post-training. We trace the failure to an objective mismatch. The KL regularization used in common policy optimization methods is evaluated on current-task data, whereas forgetting is caused by behavioral drift on prior-task distributions. To address this problem, we propose Continual Policy Optimization (CPO), a replay-free method grounded in a prior-task behavioral KL objective. CPO derives a local Fisher surrogate from the historical KL objective and uses parameter movement as a gradient-free proxy for Fisher sensitivity, enabling sparse regularization with negligible additional overhead. Experiments on three model scales and comparisons with multiple RL baselines show that CPO consistently reduces forgetting while maintaining effective adaptation and preserving broader pretrained capabilities. 
The implementation code is available at \href{https://github.com/MaolinLuo/CPO}{https://github.com/MaolinLuo/CPO}.

\end{abstract}

\section{Introduction}
\label{sec:introduction}

Vision-language models (VLMs) have demonstrated strong performance across a wide range of real-world tasks \citep{comanici2025gemini, achiam2023gpt, radford2021learning}. As VLMs are increasingly deployed in dynamic environments, they require continually adapting to new data and tasks. However, such a process is hindered by the fundamental challenge of catastrophic forgetting \citep{french1999catastrophic}, where learning new tasks degrades previously acquired capabilities. Most prior work addresses the problem under supervised fine-tuning (SFT) \citep{gu2026spectral, chen2025sefe, kurniawan2024evolving}. Recently, reinforcement learning (RL) has emerged as a promising alternative for continual post-training \citep{lai2025reinforcement}, with several studies reporting substantially less forgetting than SFT \citep{shenfeld2026rls, zhang2026why, jin2025rl, chen2025retaining}. These findings suggest an appealing possibility: \emph{is RL intrinsically resistant to forgetting in continual post-training?}

\begin{figure}[!t]
    \centering

    \begin{minipage}[t]{0.39\linewidth}
        \centering
        \vspace{0pt}
        \includegraphics[width=\linewidth]{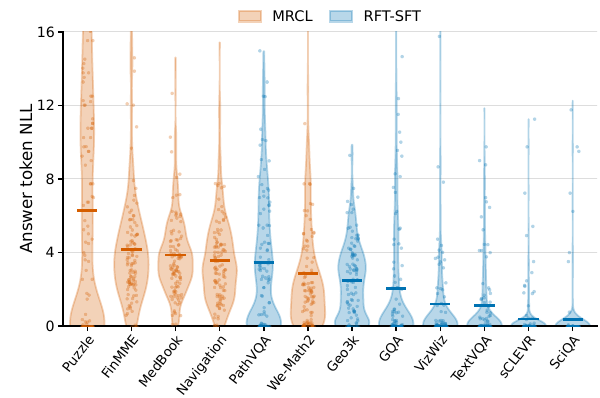}

        \captionof{figure}{
        Teacher-forced answer-token NLL across continual learning benchmarks. Results are obtained with a frozen Qwen3-VL-4B. Higher NLL indicates greater surprisal of the canonical answers to the model. Horizontal bars denote task means. Details are shown in Appendix~\ref{sec:details_nll}.
        }
        \label{fig:nll_familiarity}
    \end{minipage}
    \hfill
    \begin{minipage}[t]{0.59\linewidth}
        \centering
        \vspace{0pt}

        \begin{minipage}[b]{0.48\linewidth}
            \centering
            \includegraphics[width=\linewidth]{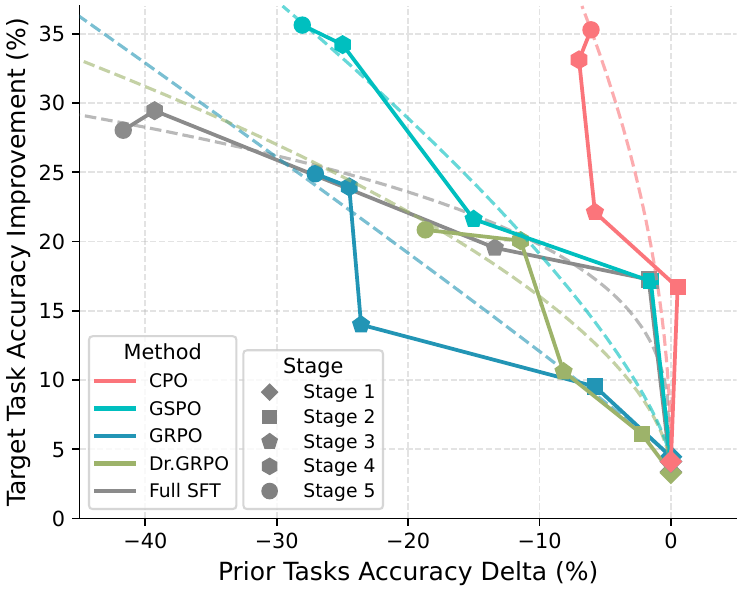}
            \par\vspace{1pt}
            {\small (a) Qwen3-VL-4B}
        \end{minipage}
        \hfill
        \begin{minipage}[b]{0.48\linewidth}
            \centering
            \includegraphics[width=\linewidth]{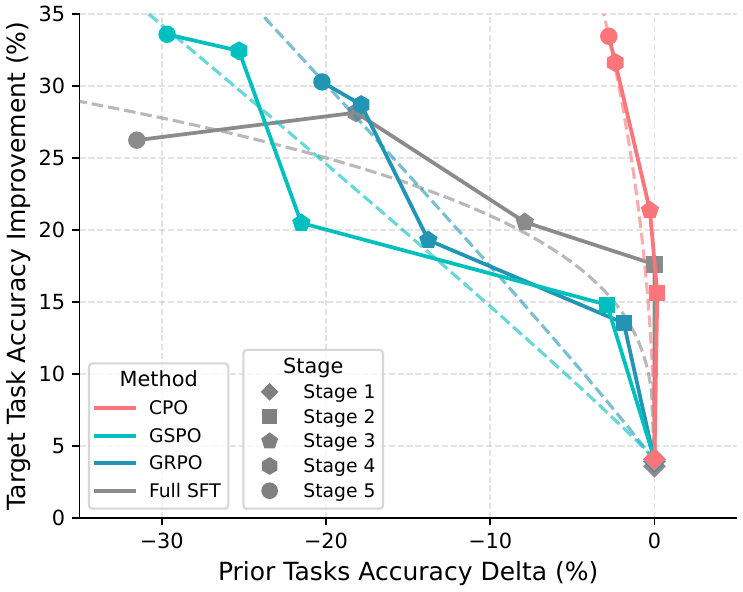}
            \par\vspace{1pt}
            {\small (b) Qwen3-VL-8B}
        \end{minipage}

        \captionof{figure}{
        Plasticity--stability Pareto frontier during continual post-training.
        The y-axis reports the cumulative mean accuracy improvement (CAI) on
        current tasks, while the x-axis reports the cumulative mean accuracy
        delta (CAD) on previously learned tasks. Definitions of CAI and CAD
        are in Appendix~\ref{sec:caicad}. Both RL and SFT exhibit catastrophic
        forgetting, whereas our method preserves prior knowledge more
        effectively while maintaining plasticity.
        }
        \label{fig:intro}
    \end{minipage}
    \vspace{-3pt}
\end{figure}

We show that the current evidence is incomplete. Existing evaluations have two blind spots that can make forgetting appear milder than it is. First, outdated benchmarks may obscure forgetting, as datasets that substantially predate the adapted model carry a higher risk of pretraining exposure. For example, \citet{lai2025reinforcement} evaluate Qwen2.5-VL \citep{bai2025qwen25vltechnicalreport} on datasets released two to seven years before the model, potentially making the corresponding capabilities easier to retain. Second, some studies evaluate RL on narrow task distributions. For instance, \citet{zhang2026why} draw conclusions from a single jigsaw task. To provide a more reliable evaluation, we introduce MRCL, a Multimodal Reasoning Continual Learning benchmark built from diverse and recent multimodal reasoning datasets. 
Its construction jointly considers release recency, teacher-forced answer-token NLL, and zero-shot performance, selecting tasks that are recent and less familiar to the base model while leaving substantial scope for adaptation.
Fig.~\ref{fig:nll_familiarity} compares MRCL with RFT-SFT, the benchmark introduced by \citet{lai2025reinforcement}, and shows that MRCL generally exhibits higher answer-token negative log likelihood (NLL). Combined with its low zero-shot performance, these observations provide complementary evidence that MRCL remains challenging for the base model. Experiments on MRCL reveal that continual RL still suffers from catastrophic forgetting, as shown in Fig.~\ref{fig:intro}, challenging the view that RL naturally solves continual learning.

We next examine why existing RL objectives fail to prevent forgetting. The key issue is an objective mismatch. \cite{shenfeld2026rls} argue that the KL regularization in GRPO \citep{shao2024deepseekmath} can directly reflect forgetting, whereas \cite{lai2025reinforcement} and \cite{chen2025retaining} suggest that KL regularization is not the main factor behind reduced forgetting. Our experiments further show that the GRPO KL constraint fails to mitigate forgetting and instead reduces plasticity on new tasks. GRPO evaluates KL divergence on current-task data, $\mathbb{E}_{x \sim \textcolor{red}{\mathcal{T}_{new}}}\left[\operatorname{KL}\left(\pi_{\theta} \| \pi_{old} \right)\right]$, which is mainly designed to prevent over-optimization \citep{shao2024deepseekmath}. Recent methods such as DAPO \citep{yu2025dapo} and GSPO \citep{zheng2025group} even remove the KL term to improve learning capacity. Catastrophic forgetting is caused by behavioral drift on previously learned tasks. Our theoretical analysis shows that the relevant constraint should be evaluated on prior-task data, $\mathbb{E}_{x \sim \textcolor{blue}{\mathcal{T}_{old}}}\left[\operatorname{KL}\left(\pi_{old} \| \pi_{\theta} \right)\right]$. Current-task KL can serve as a forgetting proxy only when the new and prior task distributions are sufficiently close. In continual learning, prior-task data are often unavailable due to efficiency constraints and data privacy concerns \citep{wang2024comprehensive, zhou2024survey, rolnick2019experience}, motivating a replay-free mechanism for constraining prior-task behavioral drift.

To constrain prior-task behavior without storing historical data, a classical approach is to regularize parameters according to their estimated importance. EWC \citep{kirkpatrick2017overcoming} and related methods \citep{aljundi2018memory, Liu_2026_CVPR}, developed primarily for supervised continual learning, estimate importance from Fisher information, parameter gradients, or optimization trajectories. Obtaining and maintaining these statistics requires additional backward computation, parameter-wise states, or training-time bookkeeping, making them difficult to scale in distributed post-training \citep{rajbhandari2020zero, narayanan2021efficient}. In our RL setting, we instead derive parameter sensitivity directly from the prior-task behavioral KL and show that its local Fisher structure can be approximated by empirical parameter movement accumulated during training. Based on the insight, CPO maintains a cumulative set of large-movement parameters and applies masked sparse movement regularization during sequential policy optimization.

In summary, our key contributions are as follows:

\begin{itemize}

\item We introduce MRCL, a continual learning benchmark combining release recency, model familiarity, and zero-shot performance to reduce prior exposure risk while retaining scope for adaptation. MRCL provides a diverse and challenging testbed for continual post-training.

\item We propose CPO, a continual RL method grounded in the prior-task behavioral KL. CPO derives a local Fisher surrogate from this objective and uses parameter movement as a gradient-free proxy for Fisher sensitivity, with negligible extra computation.

\item Extensive experiments across multiple VLM scales demonstrate that CPO substantially reduces forgetting while maintaining effective adaptation to new tasks.

\end{itemize}

\section{Multimodal Reasoning Continual Learning Benchmark}
\label{sec:benchmark}

\begin{figure}[!ht]
    \centering
    \includegraphics[width=1\linewidth]{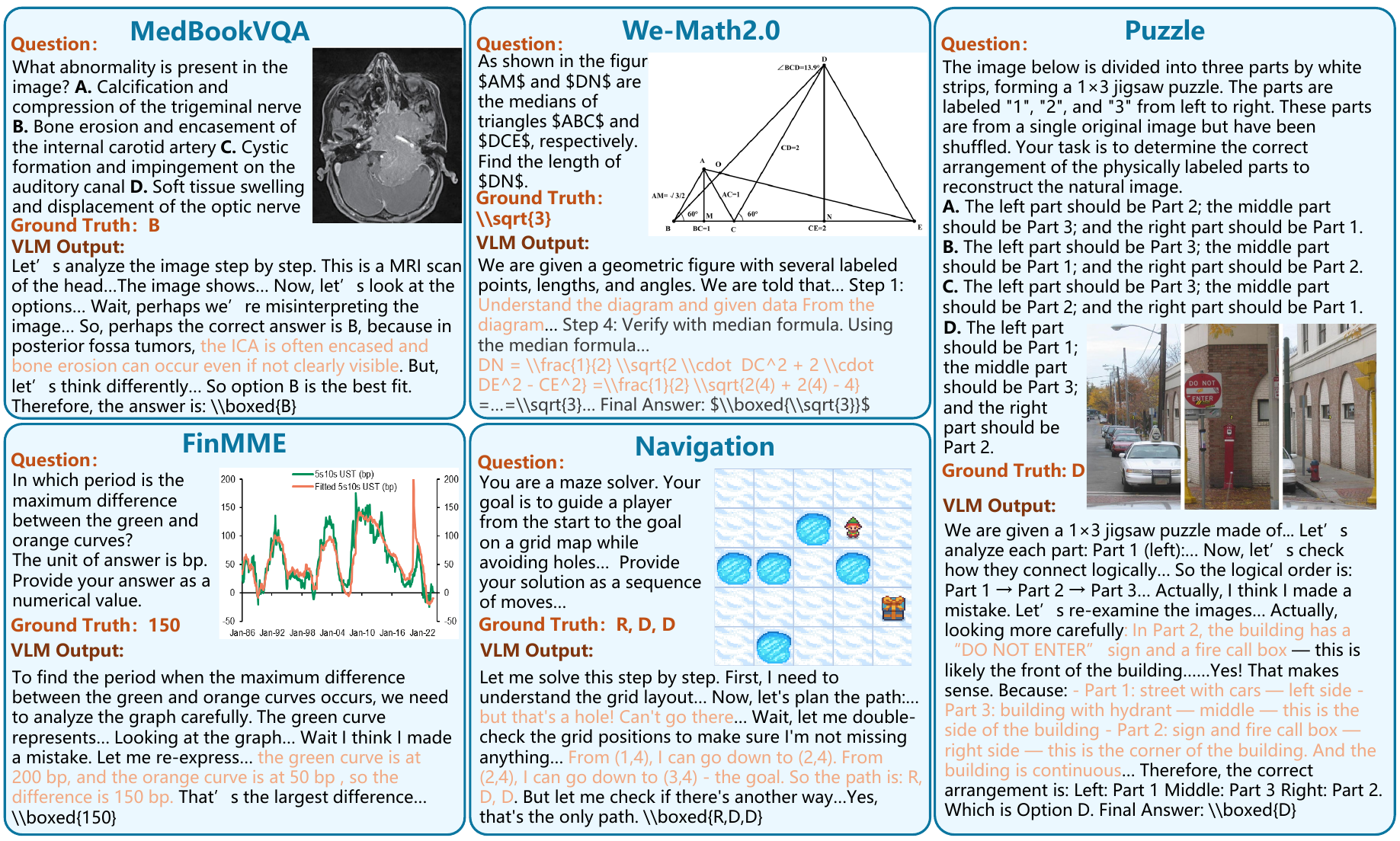}
    \caption{Overview of the MRCL benchmark. MRCL combines recently released multimodal reasoning tasks  to evaluate both the plasticity and stability of VLMs during continual post-training.}
    \label{fig:benchmark}
    
\end{figure}

To rigorously evaluate continual reinforcement learning while reducing the interference of pre-training data contamination, we construct MRCL by jointly considering benchmark recency, model familiarity, and scope for adaptation. We prioritize recently released datasets to reduce the risk that the base model has already encountered closely related data during pretraining. We further examine teacher-forced answer-token NLL to assess conditional familiarity with the canonical answer distribution, and zero-shot performance to ensure that the selected tasks leave sufficient scope for adaptation. As shown in Fig.~\ref{fig:nll_familiarity}, MRCL generally exhibits higher answer-token NLL than RFT-SFT, suggesting lower conditional-answer familiarity to the base model. MRCL covers diverse forms of multimodal reasoning, including medical understanding, navigation, geometry, visual-spatial reasoning, and financial chart analysis. Fig.~\ref{fig:benchmark} presents representative query and answer formats, and Tab.~\ref{tab:task_statistics} summarizes the basic statistics of each task.

\begin{table*}[!ht]
\centering
\caption{Basic statistics and answer formats of the evaluated tasks.}
\label{tab:task_statistics}
\begin{tabular}{wl{6em} | wc{4em} | wc{3em} | wc{3em} | wc{15em}}
\toprule
\multirow{2}{*}{Task} & \multirow{2}{*}{\shortstack{Release\\Date}} & \multirow{2}{*}{\# Train} & \multirow{2}{*}{\# Test} & \multirow{2}{*}{Answer Format} \\
 & & & & \\ %
\midrule
MedBookVQA    & 2025-06 & 1.4k & 0.6k & Multiple Choice \\
Navigation & 2025-10 &  1.5k &   0.5k & Direction Sequence \\
We-Math2.0   & 2025-08 &  5.3k & 3.7k & Numerical \& Multiple Choice \\
Puzzle     & 2025-10 &  6.0k &   0.4k & Multiple Choice \\
FinMME     & 2025-05 & 8.0k & 1.0k & Numerical \& Multi-Answer Choice \\
\bottomrule
\end{tabular}
\end{table*}

\section{Method}
\label{sec:method}

This section formalizes continual RL for VLM post-training and shows that prior-task performance degradation is bounded by behavioral KL on historical queries and old-policy prefixes. Directly evaluating the behavioral KL requires access to historical task data and old-policy rollouts. We therefore derive a parameter-space relaxation and develop a simple and effective approach named  Continual Policy Optimization.

\subsection{Continual RL Objective and Forgetting Penalty}
\label{subsec:method_setting}

We formulate continual reinforcement learning on a sequence of tasks denoted by $\{\mathcal{T}_1, \dots, \mathcal{T}_T\}$. Each task $\mathcal{T}_i$ is associated with an input distribution $\mathcal{D}_i$, where a sample $x \sim \mathcal{D}_i$ represents a multimodal query, typically comprising an image and a textual instruction. Conditioned on $x$, a policy $\pi$ autoregressively generates a response sequence $y_{1:H}$:
\begin{equation}
P_\pi(y_{1:H} \mid x) = \prod_{t=1}^H \pi(y_t \mid x, y_{<t}),
\end{equation}
where $H$ denotes the maximum generation horizon, and $t$ indexes the decoding step. Upon completion of the sequence, the environment evaluates the response quality using a terminal scalar reward $R_i(x, y_{1:H})$. The expected objective for task $\mathcal{T}_i$ under policy $\pi$ is defined as
$\eta_i(\pi) = \mathbb{E}_{x \sim \mathcal{D}_i, \, y \sim P_\pi(\cdot \mid x)}[R_i(x, y)]$.

During sequential adaptation to the $i$-th task, the learning process can be formulated as a dual-objective optimization problem that improves plasticity on new data while maintaining stability across prior tasks. Let $\pi^{(i)}$ denote the policy optimized on $\mathcal{T}_i$. The oracle continual RL objective is
\begin{equation}
\max_{\pi^{(i)}} \left[ \eta_i(\pi^{(i)}) - \lambda \sum_{k=1}^{i-1} \left( \eta_k(\pi^{(k)}) - \eta_k(\pi^{(i)}) \right) \right],
\label{eq:joint_objective}
\end{equation}
where $\lambda > 0$ indicates the critical stability-plasticity trade-off. To optimize the joint objective, we explicitly constrain the performance degradation on previous tasks, represented by the penalty term $\eta_k(\pi^{(k)}) - \eta_k(\pi^{(i)})$. In the following, we establish a theoretical surrogate for the behavioral drift.

\subsection{Prior-Task Behavioral KL Bounds Forgetting}
\label{subsec:ideal_kl}

To quantify the behavioral drift that causes catastrophic forgetting on a prior task $\mathcal{T}_{i-1}$, we analyze token-level KL divergence under the prefix visitation distribution of the old policy. We first formulate the oracle objective assuming previous-task data is accessible. Formally, let $\tilde{\pi}$ denote the old policy preserved after training on $\mathcal{T}_{i-1}$ and let $\pi$ denote the current policy optimized on $\mathcal{T}_{i}$, the autoregressive factorization induces $P_{\tilde{\pi}}(\cdot\mid x)$ and $P_{\pi}(\cdot\mid x)$, and the old-task performance is given by $\eta_{i-1}(\tilde{\pi})$.
Assuming the terminal reward function is bounded, its maximum variation range is defined as
$B_{i-1}=\sup_{x,y,y'}|R_{i-1}(x,y)-R_{i-1}(x,y')|$,
where $y$ and $y'$ represent any two possible response sequences generated for the same input query $x$.

To capture the distribution of generated prefixes at each step during the autoregressive generation, we define the normalized prefix visitation distribution under the old policy:
\begin{equation}
	d_{\tilde{\pi}}^{\mathcal{T}_{i-1}}(t,x,y_{<t})=\frac{1}{H}\mathcal{D}_{i-1}(x)P_{\tilde{\pi}}(y_{<t}\mid x),\qquad t\in\{1,\ldots,H\}.
\end{equation}
The mean token-level KL divergence over this visitation distribution characterizes the behavioral discrepancy between the two policies:
\begin{equation}
    \label{eq:oraclekl}
	\bar{D}_{\mathrm{KL}}^{d_{\tilde{\pi}}^{\mathcal{T}_{i-1}}}(\tilde{\pi},\pi)=
	\mathbb{E}_{(t,x,y_{<t})\sim d_{\tilde{\pi}}^{\textcolor{blue}{\mathcal{T}_{i-1}}}}
	\left[
		D_{\mathrm{KL}}\bigl(\tilde{\pi}(\cdot\mid x,y_{<t})\,\|\,\pi(\cdot\mid x,y_{<t})\bigr)
	\right].
\end{equation}
The above formulation differs from the KL regularization typically used in standard reinforcement learning algorithms such as GRPO, which is computed only over inputs from the current task $\textcolor{red}{\mathcal{T}_i}$. Current-task regularization can prevent forgetting only when the new data distribution is sufficiently close to the old data distribution. The KL divergence in Eq.~\ref{eq:oraclekl} instead provides a theoretical guarantee by bounding performance degradation on prior tasks, as shown in the following proposition.
\begin{proposition}[Finite horizon KL forgetting bound]
	Given a bounded reward range $B_{i-1}<\infty$, the expected return degradation on task $\mathcal{T}_{i-1}$ is bounded by the token level KL divergence:
	\begin{equation}
		\left|\eta_{i-1}(\tilde{\pi})-\eta_{i-1}(\pi)\right|
		\leq
		B_{i-1}\sqrt{\frac{H}{2}\bar{D}_{\mathrm{KL}}^{d_{\tilde{\pi}}^{\mathcal{T}_{i-1}}}(\tilde{\pi},\pi)}.
		\label{eq:forgetting_bound}
	\end{equation}
    \label{prop:kl}
\end{proposition}
The bound follows by first relating expected return degradation to total variation distance, then applying Pinsker's inequality and Jensen's inequality~\citep{cover2006elements} to control it by the expected sequence-level KL. The autoregressive KL chain rule further decomposes the sequence-level divergence into the mean token-level KL in Eq.~\ref{eq:forgetting_bound}. The detailed proof is provided in Appendix~\ref{app:kl_forgetting_bound}.

Eq.~\ref{eq:forgetting_bound} shows that forgetting on $\mathcal{T}_{i-1}$ can be controlled by keeping the current policy $\pi$ close to the old policy $\tilde{\pi}$ under the historical prefix distribution. However, directly optimizing the joint objective in Eq.~\ref{eq:joint_objective} requires access to historical distributions $\mathcal{D}_{i-1}$, while optimizing the KL constraint in Eq.~\ref{eq:oraclekl} requires sampling old-policy prefixes and forwarding $\tilde{\pi}$. For large-scale post-training, these requirements are impractical due to computational overhead and data privacy constraints.

\subsection{Fisher Relaxation of Prior-Task Behavioral KL}
\label{subsec:fisher_relaxation}

To eliminate the dependence of the oracle KL divergence on historical data, we derive a replay-free surrogate. We denote the old and active policies as $\pi_{\tilde{\theta}} \equiv  \tilde{\pi}$ and $\pi_{\theta} \equiv \pi$, where $\tilde{\theta}$ represents the frozen weights from task $\mathcal{T}_{i-1}$ and $\theta$ denotes the active weights optimized on $\mathcal{T}_i$. Since the KL divergence reaches its minimum when the two distributions are identical, we have $K_{i-1}(\tilde{\theta})=0$. The score identity implies that the first-order derivative vanishes $\nabla_\theta K_{i-1}(\theta)|_{\theta=\tilde{\theta}}=0$. Thus, defining the oracle old-task KL and expanding it around $\tilde{\theta}$ gives:
\begin{equation}
	\begin{aligned}
	K_{i-1}(\theta)
	&\coloneqq
	\bar{D}_{\mathrm{KL}}^{d_{\tilde{\pi}}^{\mathcal{T}_{i-1}}}(\pi_{\tilde{\theta}},\pi_\theta) \\
	&=
	\frac{1}{2}(\theta-\tilde{\theta})^\top H_{\tilde{\theta}}^{\mathcal{T}_{i-1}}(\theta-\tilde{\theta})
	+
	R_{i-1}^{(3)}(\theta),
	\end{aligned}
	\label{eq:taylor_with_remainder}
\end{equation}
where $H_{\tilde{\theta}}^{\mathcal{T}_{i-1}}=\nabla_\theta^2K_{i-1}(\theta)|_{\theta=\tilde{\theta}}$ denotes the Hessian matrix, and $R_{i-1}^{(3)}(\theta)$ represents the third-order Taylor remainder. For the KL divergence, the Hessian is mathematically equivalent to the Fisher Information Matrix (FIM) of the old task:
\begin{equation}
	F_{\tilde{\theta}}^{\mathcal{T}_{i-1}}
	=
	\mathbb{E}_{(t,x,y_{<t})\sim d_{\tilde{\pi}}^{\mathcal{T}_{i-1}}}
	\mathbb{E}_{a\sim\tilde{\pi}(\cdot\mid x,y_{<t})}
	\left[
		\nabla_\theta\log\pi_\theta(a\mid x,y_{<t})
		\nabla_\theta\log\pi_\theta(a\mid x,y_{<t})^\top
	\right]_{\theta=\tilde{\theta}}.
	\label{eq:fisher_definition}
\end{equation}
Thus, $H_{\tilde{\theta}}^{\mathcal{T}_{i-1}}=F_{\tilde{\theta}}^{\mathcal{T}_{i-1}}$. Omitting the higher-order remainder term defines a second-order local surrogate objective:
\begin{equation}
	Q_{i-1}(\theta;\tilde{\theta})
	=
	\frac{1}{2}(\theta-\tilde{\theta})^\top F_{\tilde{\theta}}^{\mathcal{T}_{i-1}}(\theta-\tilde{\theta}).
	\label{eq:fisher_quadratic}
\end{equation}
To reduce the prohibitive computational and memory costs of the dense FIM in large scale models, we relax the dense matrix $F_{\tilde{\theta}}^{\mathcal{T}_{i-1}}$ with its diagonal components yields: 
\begin{equation}
	Q_{i-1}^{\mathrm{diag}}(\theta;\tilde{\theta})
	=
	\frac{1}{2}\sum_{k=1}^{d}F_{kk}^{\mathcal{T}_{i-1}}(\theta_k-\tilde{\theta}_k)^2,
	\label{eq:diagonal_fisher}
\end{equation}
where the $k$-th diagonal entry is computed as:
\begin{equation}
    \label{eq:fkk}
	F_{kk}^{\mathcal{T}_{i-1}}
	=
	\mathbb{E}_{(t,x,y_{<t})\sim d_{\tilde{\pi}}^{\mathcal{T}_{i-1}}}
	\mathbb{E}_{a\sim\tilde{\pi}(\cdot\mid x,y_{<t})}
	\left[
		\left(\frac{\partial\log\pi_\theta(a\mid x,y_{<t})}{\partial\theta_k}\right)^2
	\right]_{\theta=\tilde{\theta}}.
\end{equation}
The Fisher in Eq.~\ref{eq:fkk} differs from EWC supervised empirical Fisher computed from fixed input-label pairs \citep{kirkpatrick2017overcoming}. Eq.~\ref{eq:fkk} is an old-policy token-level Fisher averaged over autoregressive prefix visitation, and thus arises from the local curvature of the prior-task behavioral KL.

Eq.~\ref{eq:diagonal_fisher} and Eq.~\ref{eq:fkk} show that the oracle KL constraint admits a local parameter-space surrogate that penalizes parameter movement according to the sensitivity of the old policy distribution. However, explicitly estimating $F_{kk}^{\mathcal{T}_{i-1}}$ still requires sampling from the old-policy prefix distribution and additional gradient computation. This additional backward computation undermines the efficiency required for scalable continual reinforcement learning, motivating a gradient-free proxy.

\subsection{Parameter Movement as a Gradient-Free Fisher Proxy}
\label{subsec:movement_proxy}

To avoid the computational cost of explicit Fisher estimation, we use empirical parameter movement during task training as a proxy for parameter importance. Let $\theta^{(i-2)}$ and $\tilde{\theta}$ denote the model checkpoints after tasks $\mathcal{T}_{i-2}$ and $\mathcal{T}_{i-1}$, respectively. The movement score for the $k$-th coordinate on task $\mathcal{T}_{i-1}$ is
\begin{equation}
	s_k^{(i-1)}=\left|\tilde{\theta}_k-\theta_k^{(i-2)}\right|.
	\label{eq:movement_score}
\end{equation}
The proxy is theoretically motivated by the relationship between policy gradient updates and Fisher sensitivity. Let $h_t=(x,y_{<t})$ denote the generation history and let $\phi_k(t,h_t,a)$ be the token-level score evaluated on the old policy:
\begin{equation}
	\phi_k(t,h_t,a)=\frac{\partial}{\partial\theta_k}\log\pi_\theta(a\mid h_t)\Big|_{\theta=\tilde{\theta}}.
\end{equation}
Specifically, the diagonal elements of the old task FIM correspond to the expectation of the squared score function:
\begin{equation}
	F_{kk}^{\mathcal{T}_{i-1}}=\mathbb{E}_{(t,h_t)\sim d_{\tilde{\pi}}^{\mathcal{T}_{i-1}}}
	\left[
		\mathbb{E}_{a\sim\tilde{\pi}(\cdot\mid h_t)}[\phi_k(t,h_t,a)^2]
	\right].
\end{equation}
Since policy gradient updates are proportional to reward-weighted scores, parameters with larger Fisher sensitivity tend to exhibit larger accumulated movement during training.
\begin{proposition}[Parameter movement as a noisy Fisher proxy]
\label{subsec:proposition2}
    Under locally stationary approximation with learning rate $\eta$ for stochastic updates, the expected squared movement satisfies:
	\begin{equation}
		\mathbb{E}\left[\left|\tilde{\theta}_k-\theta_k^{(i-2)}\right|^2\right]
		=
        \eta^2\sigma_A^2 C_k F_{kk}^{\mathcal{T}_{i-1}}.
        \label{eq:movement_fisher_relation}
    \end{equation}
    where $\sigma_A^2$ is the second moment of the advantage and $C_k$ is a coordinate coefficient reflecting both reward-score alignment and stochastic-gradient variance. 
    Consequently, when $C_k$ is approximately uniform across coordinates, ranking by the expected squared movement $\mathbb{E}[(s_k^{(i-1)})^2]$ is consistent with ranking by $F_{kk}^{\mathcal{T}_{i-1}}$. 
\end{proposition}

We defer the rigorous proof and the explicit formulation of $C_k$ to
Appendix~\ref{app:movement_fisher_proxy}. Proposition~\ref{subsec:proposition2}
shows that empirical parameter movement approximately preserves Fisher importance
ordering when $C_k$ varies mildly. Appendix~\ref{cpoc:sec} further shows that
approximate uniformity is not necessary. It is sufficient that a substantial
fraction of the Fisher-important coordinates have sufficiently large $C_k$,
while $C_k$ may vary arbitrarily elsewhere.

\begin{algorithm}[t]
	\caption{Continual Policy Optimization}
	\label{alg:cumulative_mask}
	\begin{algorithmic}[1]
		\REQUIRE Task sequence $\{\mathcal{T}_i\}_{i=1}^{T}$, top ratio $p$, regularization strength $\lambda$
		\STATE Initialize protected set $\mathcal{S}^{(0)} \leftarrow \varnothing$ and checkpoint $\theta^{(0)}$
		\FOR{$i=1$ to $T$}
		\STATE Initialize $\tilde{\theta} \leftarrow \theta^{(i-1)}$
		\STATE Train on $\mathcal{T}_i$ with Eq.~\ref{eq:l1_loss_vector} using protected set $\mathcal{S}^{(i-1)}$
		\STATE Save the task-end checkpoint $\theta^{(i)}$
		\STATE Add the top $p$ moving coordinates to obtain $S^{(i)}$ with Eq.~\ref{eq:mask_update}
		\ENDFOR
	\end{algorithmic}
\end{algorithm}

\subsection{Cumulative Masked Movement Regularization}
\label{subsec:cumulative_mask}

CPO maintains a cumulative protected set $\mathcal{S}^{(i-1)}$ before training task $\mathcal{T}_i$, initialized as $\mathcal{S}^{(0)}=\varnothing$. Let $\tilde{\theta}\equiv\theta^{(i-1)}$ be the previous checkpoint and let $m^{(i-1)}$ be the binary mask of $\mathcal{S}^{(i-1)}$. We optimize the current policy with RL loss $\mathcal{L}^{\text{RL}}_i(\theta)$ and masked $L_1$ movement regularization:
\begin{equation}
    \mathcal{L}_i(\theta) = \mathcal{L}^{\text{RL}}_i(\theta) + \frac{\lambda}{\|m^{(i-1)}\|_0} \left\| m^{(i-1)} \odot (\theta - \tilde{\theta}) \right\|_1.
    \label{eq:l1_loss_vector}
\end{equation}
When $\mathcal{S}^{(i-1)}$ is empty, the regularization term is omitted.
We employ $L_1$ regularization because the movement proxy provides ranking information rather than calibrated Fisher magnitudes. While the Fisher regularizer in Eq.~\ref{eq:diagonal_fisher} assigns continuous parameter weights, top-$p$ movement selection yields a binary support. An unweighted masked $L_2$ regularizer treats all protected parameters equally and can still allow dense parameter drift. Masked $L_1$ regularization promotes sparse movement, encouraging most selected parameters to remain close to the previous checkpoint.
After task $T_i$, we rank parameter movement independently within each trainable parameter and add the top $p$ moving coordinates from each tensor to the protected set. Let $\mathcal{G}=\{G_1,\ldots,G_M\}$ denote the partition of trainable coordinates induced by the $M$ parameter tensors. We update the protected set as
\begin{equation}
S^{(i)}
=
S^{(i-1)}
\cup
\bigcup_{G_m\in\mathcal{G}}
\operatorname{TopP}_{p}
\left(
\left\{
\left|
\theta_k^{(i)}-\tilde{\theta}_k
\right|
\right\}_{k\in G_m}
\right).
\label{eq:mask_update}
\end{equation}
We summarize the overall training procedure in Algorithm~\ref{alg:cumulative_mask}.

\section{Experiments}
\label{sec:experiments}

\subsection{Experimental Setup}

\textbf{Evaluation Metrics.} We report three aggregated accuracy metrics and final accuracy $a_{T,j}$ of each task. Let $T$ denote the total number of tasks and $a_{i,j}$ represent the accuracy on task $j$ evaluated after the model completes training on task $i$. Mean Finetune Accuracy (MFT) measures the average performance on each task when evaluated immediately after training, aiming to quantify the ability to learn new tasks. Mean Final Accuracy (MFN) measures the average accuracy at the end of the training sequence, reflecting the severity of forgetting. Mean Task Accuracy (MTA) provides a holistic measure of performance stability by averaging the accuracy throughout learning process.
\begin{equation}
    \text{MFT} = \frac{1}{T} \sum_{i=1}^T a_{i,i}, \quad \text{MFN} = \frac{1}{T} \sum_{j=1}^T a_{T,j}, \quad \text{MTA} = \frac{1}{T} \sum_{j=1}^T \left( \frac{1}{T-j+1} \sum_{i=j}^T a_{i,j} \right).
\end{equation}
\textbf{Reward Design.} The reward consists of an accuracy reward and a length penalty. The accuracy reward aligns exactly with the aforementioned evaluation scoring rules for most tasks, except for \textit{Navigation}, where we provide dense stepwise feedback by awarding $1/k$ points for each correctly predicted step as $k$ is the total number of steps in the ground truth path. To prevent the policy from collapsing into degenerate short responses, we apply a length penalty of $-0.5$ if the generated output contains fewer than $20$ words.

\subsection{Continual Learning Performance on MRCL}
\label{sec:main_results}

\begin{table*}[!tb]
\centering
\setlength{\tabcolsep}{2pt}
\caption{Comparison of our proposed method with baseline approaches on Qwen3-VL-4B.}
\label{tab:comparison_4B}
\begin{tabular}{l | *{5}{wc{3.6em}} | *{3}{wc{3em}}}
\toprule
\multirow{2}{*}{Method} & \multicolumn{5}{c|}{\textbf{Final Accuracy on Each Task (\%)}} & \multicolumn{3}{c}{\textbf{Aggregated Acc. (\%)}} \\
 & MedBook & Naviga. & We-Math2 & Puzzle & FinMME & MFT$\uparrow$ & MFN$\uparrow$ & MTA$\uparrow$ \\
\midrule
Zero-shot$^1$ & 57.39 & 3.31 & 35.38 & 36.00 & 41.60 & -- & -- & 34.74 \\
Full SFT &  52.12 &	5.70 &	37.37 &	71.75 &	34.55 &	62.77 &	40.30 &	50.20   \\
\midrule
Zero-shot$^2$ & 48.73 &	20.77 &	39.06 &	24.75 &	38.95 & -- & -- & 34.45 \\
GRPO* \tiny{\citep{shao2024deepseekmath}} & 58.23 &	32.72 &	\underline{62.20} &	60.25 &	\underline{47.35} &	58.67 &	52.15 &	55.06  \\
GRPO \tiny{\citep{shao2024deepseekmath}} & \underline{60.61} &	15.26 &	56.21 &	66.50 &	43.70 &	59.34 &	48.46 &	51.85   \\
Dr.GRPO \tiny{\citep{liu2025understanding}} & 54.84 &	11.03 &	56.03 &	58.75 &	42.85 & 55.29 &	44.70 &	50.07  \\
GSPO \tiny{\citep{zheng2025group}} & 58.40 	& \underline{43.20} &	48.82 &	\textbf{83.25} &	46.00 &	\textbf{70.07} &	\underline{55.93} &	\underline{61.82}   \\
CPO (Ours) & \textbf{68.93} &	\textbf{73.71} &	\textbf{65.66} &	\underline{78.50} &	\textbf{49.80} &	\underline{69.73} &	\textbf{67.32} &	\textbf{67.69}   \\
\bottomrule
\end{tabular}
\end{table*}

We evaluate all methods over a five-stage continual learning sequence, with each dataset treated as one task. As shown in Tab.~\ref{tab:comparison_4B}, \ref{tab:comparison_2B}, and \ref{tab:comparison_8B}, CPO achieves the highest MFN and MTA across all three model scales while maintaining competitive MFT. On Qwen3-VL-4B, CPO reaches $67.32\%$ MFN, exceeding the highest competing RL baseline by $11.39\%$. The same trend holds on the 2B and 8B models, indicating that the retention gains persist across model scales. GRPO$^{*}$, which uses current-task KL regularization, consistently lowers MFT without yielding a consistent final accuracy gain, suggesting that such regularization limits adaptation without reliably improving retention.

\subsection{Pretrained Capability Retention on External Benchmarks}

We further evaluate pretrained capability retention after MRCL training on diverse multimodal benchmarks. As shown in Tab.~\ref{tab:oodbenchmark_4b}, \ref{tab:oodbenchmark_2b}, and \ref{tab:oodbenchmark_8b}, CPO consistently preserves pretrained capabilities better than competing post-training methods across model scales. On Qwen3-VL-4B, CPO achieves an average score of $64.30\%$, compared with $50.90\%$ for GSPO and $59.50\%$ for the pretrained model. The advantage extends across the 2B and 8B models and multiple capability dimensions, indicating broad retention of the pretrained capability profile with further gains on several benchmarks. We hypothesize that these gains reflect cumulative positive transfer across continual stages, as transferable capabilities acquired on each task can accumulate only when they survive subsequent updates. By suppressing cross-task behavioral drift, CPO better preserves and accumulates such knowledge across the training stream, an interpretation further supported by the substantially lower prior-task behavioral KL in Fig.~\ref{fig:behavioral_kl}.

\begin{table}[!h]
\centering
\setlength{\tabcolsep}{2pt}
\caption{Comparison of Qwen3-VL-4B multimodal performance under different fine-tuning settings across diverse benchmarks.}
\label{tab:oodbenchmark_4b}
\begin{tabular}{
    >{\raggedright\arraybackslash}p{2.3cm} >{\raggedright\arraybackslash}p{2.5cm} | >{\centering\arraybackslash}p{1.2cm} | *{5}{>{\centering\arraybackslash}p{1.2cm}}
}
\toprule
\multirow{2}{*}[-1.0ex]{\textbf{Dimension}} & \multicolumn{1}{l}{\multirow{2}{*}[-1.0ex]{\textbf{Benchmark}}} & \multicolumn{6}{c}{\textbf{Qwen3-VL-4B}} \\
\cmidrule{3-8}
 & & \textbf{Base} & \textbf{SFT}  & \textbf{GRPO} & \textbf{Dr.GR} & \textbf{GSPO} & \textbf{CPO} \\
\midrule
\multirow{4}{*}{STEM} 
    & MMMU\textsubscript{pro\_full}    &  51.20 &	24.76 &	39.74 &	\underline{41.10} &	35.49 &	\textbf{50.04}  \\
    & MathVista\textsubscript{mini}    & 56.90 &	51.30 &	\underline{65.10} &	58.80 &	62.80 &	\textbf{73.90}   \\
    & MathVision\textsubscript{TEST}   & 33.78 &	18.52 &	\underline{31.81} &	30.33 &	23.75 &	\textbf{47.07}   \\
    & MathVerse\textsubscript{mini}    &  45.61 &	29.14 &	\underline{47.64} &	43.81 &	40.20 &	\textbf{57.89}  \\
\cline{1-1} 
\multirow{2}{*}{General VQA} & RealworldQA    &  68.24 &	51.24 &	47.58 &	\textbf{61.44} &	\underline{59.87} &	49.41  \\
    & MMStar       &  58.93 &	54.73 &	\underline{61.20} &	57.93 &	61.07 &	\textbf{67.27}  \\
\cline{1-1}
Hallucination & POPE       &  87.16 &	\underline{84.57} &	83.73 &	83.76 &	80.75 &	\textbf{87.08}  \\
\cline{1-1}
\multirow{3}{*}{\shortstack[l]{Chart/Doc\\Understanding}} 
    & DocVQA\textsubscript{VAL}       &  65.04 &	48.40 &	54.78 &	\underline{60.03} &	55.04 &	\textbf{75.88}  \\
    & CharXiv\textsubscript{(DQ)}   &  71.63 &	40.62 &	\underline{53.67} &	50.92 &	48.33 	& \textbf{75.25}  \\
    & CharXiv\textsubscript{(RQ)}   &  33.10 &	23.70 &	28.70 &	\underline{28.80} &	24.60 &	\textbf{38.20}  \\
\cline{1-1}
Grounding & CountBenchQA & 82.89 &	66.80 &	\underline{70.06} &	53.16 &	68.02 &	\textbf{85.34}   \\
\midrule
\multicolumn{2}{c|}{\textbf{Average}} &  59.50 &	44.89 &	\underline{53.09} &	51.83 &	50.90 &	\textbf{64.30}  \\
\bottomrule
\end{tabular}
\end{table}

\subsection{Prior-Task Behavioral Drift}
\label{sec:behavioral_kl}

Proposition~\ref{prop:kl} identifies behavioral drift on prior-task distributions as the quantity directly associated with forgetting.
After each stage $T_i$, we evaluate behavioral KL between the current policy and the task-end policy of each preceding task $T_j$ where $j<i$. We sample queries from each prior task, generate autoregressive prefixes with its task-end policy, and compute the mean token-level KL. The resulting KL values are averaged across all preceding tasks. As shown in Fig.~\ref{fig:behavioral_kl}, CPO yields substantially lower average prior-task behavioral KL than GSPO across continual stages. The difference becomes more pronounced in later stages, while CPO maintains consistently low behavioral drift. 

\begin{figure}[!ht]
    \centering

    \begin{minipage}[t]{0.49\textwidth}
        \vspace{0pt}
        \centering

        \begin{subfigure}[t]{0.48\linewidth}
            \centering
            \includegraphics[width=\linewidth]
            {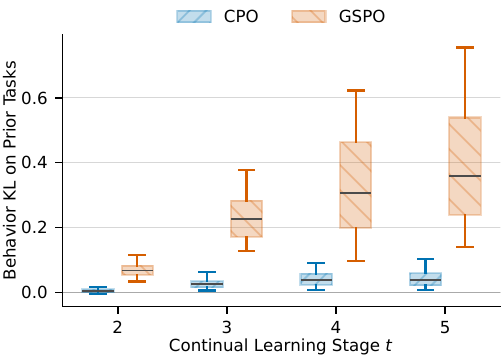}
            \caption{Qwen3-VL-4B}
            \label{fig:kl4B}
        \end{subfigure}
        \hfill
        \begin{subfigure}[t]{0.48\linewidth}
            \centering
            \includegraphics[width=\linewidth]
            {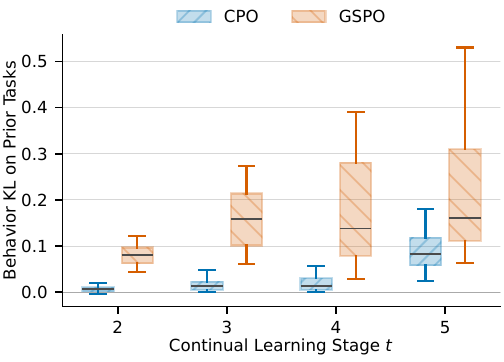}
            \caption{Qwen3-VL-8B}
            \label{fig:kl8B}
        \end{subfigure}

        \captionof{figure}{
            Prior-task behavioral KL averaged across all preceding tasks at each continual learning stage on Qwen3-VL-4B and Qwen3-VL-8B. Lower values indicate smaller behavioral drift on previously learned tasks.
        }
        \label{fig:behavioral_kl}
    \end{minipage}
    \hfill
    \begin{minipage}[t]{0.47\textwidth}
        \vspace{0pt}
        \centering

        \captionof{table}{
            Ablation of protected-parameter selection and
regularization in CPO on Qwen3-VL-4B.
Random replaces movement-based selection with a random
mask of the same size, while $L_2$ replaces masked
$L_1$ regularization.
        }
        \label{tab:cpo_ablation}

        \setlength{\tabcolsep}{2pt}
        \small

        \begin{tabular}{l | *{3}{wc{3.0em}}}
            \toprule
            \textbf{Method}
            & \textbf{MFT$\uparrow$}
            & \textbf{MFN$\uparrow$}
            & \textbf{MTA$\uparrow$} \\
            \midrule
            GSPO
            & \textbf{70.07} &	55.93 &	61.82   \\
            CPO (random $L_1$ mask)
            & 69.22 &	56.15 &	62.05   \\
            CPO (top $L_2$ mask)
            & 69.68 &	\underline{57.36} &	\underline{64.66}   \\
            CPO (top $L_1$ mask)
            & \underline{69.73} &	\textbf{67.32} &	\textbf{67.69}   \\
            \bottomrule
        \end{tabular}

    \end{minipage}

\end{figure}

\subsection{Protected-Parameter Selection and Regularization}
\label{sec:ablation}

Tab.~\ref{tab:cpo_ablation} examines the importance of
movement-based parameter selection and masked $L_1$
regularization. Random selection brings little improvement
over GSPO, whereas movement-based selection substantially
improves retention. This suggests that parameter movement
provides a meaningful signal for identifying parameters
important to previously learned behavior. Replacing masked
$L_1$ with $L_2$ also weakens retention. Masked $L_2$
allows small changes across many protected parameters.
These changes can accumulate and lead to forgetting.
We further analyze this behavior in Appendix~\ref{sec:suppl_exp}.
CPO achieves the highest MFN and MTA while maintaining
competitive MFT.

\section{Related Work}
\label{sec:related works}

Continual post-training seeks to adapt VLMs to evolving tasks without erasing previously acquired capabilities \citep{yu2026recent, wang2024comprehensive}. Research has largely focused on forgetting under SFT, while recent work has increasingly examined the performance of RL in continual learning.

\textbf{Reinforcement learning} optimizes models through reward signals and has become central to eliciting deliberative reasoning capabilities \citep{oh2025discovering, wu2026generalization}. Policy optimization has evolved from PPO \citep{schulman2017proximal} to GRPO \citep{shao2024deepseekmath} and following methods such as DAPO \citep{yu2025dapo} and GSPO \citep{zheng2025group}, motivating interest in continual RL \citep{shenfeld2026rls,abel2023definition}. A series of recent studies \citep{shenfeld2026rls,zhang2026why,chen2025retaining,lai2025reinforcement} report that RL is more resistant to catastrophic forgetting than SFT, and \citet{jin2025rl} further show that RL fine-tuning can recover out-of-distribution capabilities degraded by SFT. The mechanism behind this robustness remains debated. \citet{shenfeld2026rls} attribute the effect partly to the KL regularization used in standard policy optimization, whereas \citet{chen2025retaining} and \citet{lai2025reinforcement} find that KL regularization plays a limited role in retention. Other studies instead emphasize on-policy data generation \citep{zhang2026why,chen2025retaining}, while \citet{lou2026overcoming} shows that RL can still suffer from non-negligible forgetting on classification. Despite these observations, catastrophic forgetting in multi-modal reasoning remains underexplored, and the role of policy regularization in continual RL is still not well understood.

\textbf{Replay-free continual learning} mitigates forgetting by estimating parameter importance and constraining subsequent updates. EWC \citep{kirkpatrick2017overcoming} uses Fisher information to identify important parameters,  requiring an additional gradient-based estimation over task data. MAS \citep{aljundi2018memory} similarly estimates importance from output sensitivity via parameter gradients. SI \citep{zenke2017continual} and RWalk \citep{chaudhry2018riemannian} avoid separate importance estimation by accumulating gradient-dependent statistics during optimization. More recently,  \cite{luo2026keeplora, liang2024inflora, saha2021gradient} constrain  adaptation through protected subspaces constructed from task gradients. Although these methods avoid storing historical data, importance estimation still requires extra backward computation, parameter-wise statistics, or training-time state tracking beyond standard task optimization. These overheads increase the difficulty of integration into distributed training, where memory, computation, and communication are major scaling constraints \citep{rajbhandari2020zero,narayanan2021efficient}, limiting the practicality for large model training.

\section{Conclusion}
\label{sec:conclusion}

In this work, we challenge the view that reinforcement learning is inherently resistant to catastrophic forgetting in continual VLM post-training. Designed around release recency, model familiarity, and scope for adaptation, MRCL reveals catastrophic forgetting across RL methods. We further identify an objective mismatch between current-task KL regularization and the prior-task behavioral KL that bounds performance degradation. Motivated by this mismatch, CPO derives a local Fisher surrogate for the prior-task behavioral KL and uses parameter movement as a gradient-free proxy for Fisher sensitivity in the RL setting with negligible additional overhead.  MRCL and CPO move continual RL beyond presumed robustness by exposing forgetting under challenging evaluation and explicitly constraining prior-task policy drift, providing a simple and scalable path for continual VLM post-training. This perspective informs continual learning benchmark construction and supports an effective approach to continual RL, grounded in both theory and empirical evidence.

\bibliography{iclr2027_conference}
\bibliographystyle{iclr2027_conference}

\newpage

\appendix
\section{Performance on Qwen3-VL-2B and Qwen3-VL-8B}

This section reports the complete results on Qwen3-VL-2B and Qwen3-VL-8B, complementing the Qwen3-VL-4B results in the main text. As shown in Tab.~\ref{tab:comparison_2B} and Tab.~\ref{tab:comparison_8B}, CPO achieves the highest MFN and MTA on both model scales while maintaining competitive MFT. Tab.~\ref{tab:oodbenchmark_2b} and Tab.~\ref{tab:oodbenchmark_8b} further show that CPO better preserves pretrained capabilities than competing post-training methods across external benchmarks. These results confirm that the effectiveness of CPO persists across model scales and evaluation settings.

\begin{table*}[!ht]
\centering
\setlength{\tabcolsep}{2pt}
\caption{Comparison of our proposed method with baseline approaches on Qwen3-VL-2B.}
\label{tab:comparison_2B}
\begin{tabular}{l | *{5}{wc{3.6em}} | *{3}{wc{3em}}}
\toprule
\multirow{2}{*}{Method} & \multicolumn{5}{c|}{\textbf{Final Accuracy on Each Task (\%)}} & \multicolumn{3}{c}{\textbf{Aggregated Acc. (\%)}} \\
 & MedBook & Naviga. & We-Math2 & Puzzle & FinMME & MFT$\uparrow$ & MFN$\uparrow$ & MTA$\uparrow$ \\
\midrule
Zero-shot$^1$ & 48.73 & 0.00 & 22.25 & 35.75 & 32.75 & -- & -- & 27.90 \\
Full SFT & 35.65 &	2.76 &	29.07 	& 78.50 &	30.65 &	57.53 &	35.33 &	45.49    \\
\midrule
Zero-shot$^2$ & 48.05 &	1.65 &	41.28 &	22.00 &	31.45 & -- & -- & 28.89 \\
GRPO* \tiny{\citep{shao2024deepseekmath}} & 54.16 &	4.41 &	46.56 &	64.99 &	41.50 &	48.82 &	42.32 &	44.96  \\
GRPO \tiny{\citep{shao2024deepseekmath}} & 52.29 &	17.10 &	\underline{50.32} &	60.25 &	41.50 &  55.90 &	44.29 &	49.33  \\
GSPO \tiny{\citep{zheng2025group}} & \underline{55.52} &	\underline{20.59} &	47.07 &	\underline{84.00} &	\underline{42.25} &	\underline{64.70} &	\underline{49.89} &	\underline{56.34}    \\
CPO (Ours) & \textbf{59.76} &	\textbf{23.71} &	\textbf{53.71} &	\textbf{87.25} &	\textbf{44.65} &	\textbf{65.34} &	\textbf{53.82} &	\textbf{58.85}  \\
\bottomrule
\end{tabular}
\end{table*}

\begin{table*}[!ht]
\centering
\setlength{\tabcolsep}{2pt}
\caption{Comparison of our proposed method with baseline approaches on Qwen3-VL-8B.}
\label{tab:comparison_8B}
\begin{tabular}{l | *{5}{wc{3.6em}} | *{3}{wc{3em}}}
\toprule
\multirow{2}{*}{Method} & \multicolumn{5}{c|}{\textbf{Final Accuracy on Each Task (\%)}} & \multicolumn{3}{c}{\textbf{Aggregated Acc. (\%)}} \\
 & MedBook & Naviga. & We-Math2 & Puzzle & FinMME & MFT$\uparrow$ & MFN$\uparrow$ & MTA$\uparrow$ \\
\midrule
Zero-shot$^1$ & 61.46 &	0.00 &	36.20 &	36.50 &	\underline{47.60}  & -- & -- & 36.35  \\
Full SFT & \underline{62.31} &	4.60 &	43.48 &	64.50 &	38.05 &	62.59 &	42.59 &	54.11  \\
\midrule
Zero-shot$^2$ & 51.95 &	25.55 &	36.49 &	25.25 &	39.65  & -- & -- & 35.78 \\
GRPO \tiny{\citep{shao2024deepseekmath}} & 61.29 &	\underline{42.83} &	\underline{59.87} &	71.25 &	47.55 &	66.06 &	\underline{56.56} &	60.46  \\
GSPO \tiny{\citep{zheng2025group}} &  61.12 &	40.07 &	58.96 &	\underline{72.25} &	45.40 &	\textbf{69.36} &	55.56 &	\underline{60.61}   \\
CPO (Ours) & \textbf{71.82} &	\textbf{79.60} &	\textbf{64.45} &	\textbf{72.75} &	\textbf{48.80} &	\underline{69.23} &	\textbf{67.48} &	\textbf{68.20}   \\
\bottomrule
\end{tabular}
\end{table*}

\begin{table}[!ht]
\centering
\caption{Comparison of Qwen3-VL-2B multimodal performance under different fine-tuning settings across diverse benchmarks.}
\label{tab:oodbenchmark_2b}
\begin{tabular}{ll  | >{\centering\arraybackslash}p{1.1cm} | *{4}{>{\centering\arraybackslash}p{1.1cm}}}
\toprule
\multirow{2}{*}[-1.0ex]{\textbf{Dimension}} & \multicolumn{1}{l}{\multirow{2}{*}[-1.0ex]{\textbf{Benchmark}}} & \multicolumn{5}{c}{\textbf{Qwen3-VL-2B}} \\
\cmidrule{3-7}
 & & \textbf{Base} & \textbf{SFT} & \textbf{GRPO} & \textbf{GSPO} & \textbf{CPO} \\
\midrule
\multirow{4}{*}{STEM} 
    & MMMU\textsubscript{pro\_full}    & 33.74 &	21.44 &	\underline{31.34} &	27.98 &	\textbf{33.49}  \\
    & MathVista\textsubscript{mini}    & 48.00 &	37.10 &	49.00 &	\underline{53.20} &	\textbf{60.50}  \\
    & MathVision\textsubscript{TEST}   & 14.87 &	16.41 &	\underline{20.62} &	20.23 &	\textbf{23.09}  \\
    & MathVerse\textsubscript{mini}    & 31.12 &	22.03 &	35.43 &	\underline{35.69} &	\textbf{40.46}  \\
\cline{1-1}
\multirow{2}{*}{General VQA} 
    & RealworldQA  & 57.78 &	47.06 &	54.25 &	\underline{57.39} &	\textbf{58.56}  \\
    & MMStar       & 52.80 &	46.20 &	\underline{52.80} &	52.47 	& \textbf{54.60}  \\
\cline{1-1}
Hallucination & POPE       & 85.61 &	81.90 &	81.95 &	\underline{83.47} &	\textbf{83.73}  \\
\cline{1-1}
\multirow{3}{*}{\shortstack[l]{Chart/Document\\Understanding}} 
    & DocVQA\textsubscript{VAL}       & 56.50 &	47.84 &	44.83 &	\underline{50.55} &	\textbf{59.34}  \\
    & CharXiv\textsubscript{(DQ)}   & 60.98 &	35.43 &	40.65 &	\underline{41.05} &	\textbf{46.06}  \\
    & CharXiv\textsubscript{(RQ)}   & 26.70 &	16.80 &	\underline{22.80} &	22.40 &	\textbf{24.90}  \\
\cline{1-1}
Grounding & CountBenchQA & 83.30 &	13.24 &	49.69 &	\underline{63.54} &	\textbf{65.38}  \\
\midrule
\multicolumn{2}{c|}{\textbf{Average}} & 50.13 &	35.04 &	43.94 &	\underline{46.18} &	\textbf{50.01}  \\
\bottomrule
\end{tabular}
\end{table}

\begin{table}[!ht]
\centering
\caption{Comparison of Qwen3-VL-8B multimodal performance under different fine-tuning settings across diverse benchmarks.}
\label{tab:oodbenchmark_8b}
\begin{tabular}{ll | >{\centering\arraybackslash}p{1.1cm} | *{4}{>{\centering\arraybackslash}p{1.1cm}}}
\toprule
\multirow{2}{*}[-1.0ex]{\textbf{Dimension}} & \multicolumn{1}{l}{\multirow{2}{*}[-1.0ex]{\textbf{Benchmark}}} & \multicolumn{5}{c}{\textbf{Qwen3-VL-8B}} \\
\cmidrule{3-7}
 & & \textbf{Base} & \textbf{SFT} & \textbf{GRPO} & \textbf{GSPO} & \textbf{CPO} \\
\midrule
\multirow{4}{*}{STEM} 
    & MMMU\textsubscript{pro\_full}    & 56.71 &	27.19 &	\underline{44.36} &	42.16 &	\textbf{52.45}  \\
    & MathVista\textsubscript{mini}    & 50.80 &	55.30 &	\underline{66.20} &	64.00 &	\textbf{74.00}  \\
    & MathVision\textsubscript{TEST}   & 36.78 &	20.92 &	\underline{37.20} &	31.97 &	\textbf{49.18}  \\
    & MathVerse\textsubscript{mini}    & 43.17 &	32.89 &	\underline{49.24} &	47.03 &	\textbf{57.39}  \\
\cline{1-1} 
\multirow{2}{*}{General VQA} 
    & RealworldQA  & 66.93 &	\textbf{61.83} &	42.35 &	\underline{61.18} & 	\textbf{61.83}  \\
    & MMStar       & 55.87 &	58.60 &	\underline{63.47} &	61.67 &	\textbf{69.93}  \\
\cline{1-1}
Hallucination & POPE       & 76.73 &	83.26 &	84.54 &	\underline{85.02} &	\textbf{85.52}  \\
\cline{1-1}
\multirow{3}{*}{\shortstack[l]{Chart/Document\\Understanding}} 
    & DocVQA\textsubscript{VAL}       & 68.16 &	\textbf{66.12} &	43.30 &	39.99 &	\underline{62.46}  \\
    & CharXiv\textsubscript{(DQ)}   & 82.00 &	49.48 &	\underline{65.70} &	52.52 &	\textbf{81.33}  \\
    & CharXiv\textsubscript{(RQ)}   &  32.70 &	29.50 &	\underline{32.40} &	30.40 &	\textbf{39.00}  \\
\cline{1-1}
Grounding & CountBenchQA & 86.97 &	75.36 &	\underline{76.58} &	74.95 &	\textbf{82.28}  \\
\midrule
\multicolumn{2}{c|}{\textbf{Average}} & 59.71 &	50.95 &	\underline{55.03} &	53.72 &	\textbf{65.03}  \\
\bottomrule
\end{tabular}
\end{table}

\section{Theoretical Proofs}
\label{app:theory}

\subsection{Finite-Horizon Terminal-Reward KL Forgetting Bound}
\label{app:kl_forgetting_bound}

\begin{proposition}[Finite horizon KL forgetting bound]
\label{prop:app_kl_forgetting_bound}
Consider an old task $\mathcal{T}_{i-1}$ with prompts $x\sim\mathcal{D}_{i-1}$, finite-horizon autoregressive generations $y_{1:H}$, and terminal reward $R_{i-1}(x,y_{1:H})$. Let $\tilde{\pi}$ be the policy obtained after training on $\mathcal{T}_{i-1}$ and let $\pi$ be the policy after subsequent adaptation. Their sequence distributions are
\begin{equation}
    P_{\tilde{\pi}}(y_{1:H}\mid x)=\prod_{t=1}^{H}\tilde{\pi}(y_t\mid x,y_{<t}),
    \qquad
    P_{\pi}(y_{1:H}\mid x)=\prod_{t=1}^{H}\pi(y_t\mid x,y_{<t}).
\end{equation}
Define the task return
\begin{equation}
    \eta_{i-1}(\pi)=\mathbb{E}_{x\sim\mathcal{D}_{i-1},\,y\sim P_{\pi}(\cdot\mid x)}[R_{i-1}(x,y)]
\end{equation}
and assume the terminal reward has bounded range
\begin{equation}
    B_{i-1}=\sup_{x,y,y'}|R_{i-1}(x,y)-R_{i-1}(x,y')|<\infty.
\end{equation}
Let the normalized old-policy prefix visitation distribution be
\begin{equation}
    d_{\tilde{\pi}}^{\mathcal{T}_{i-1}}(t,x,y_{<t})=\frac{1}{H}\mathcal{D}_{i-1}(x)P_{\tilde{\pi}}(y_{<t}\mid x),
    \qquad t\in\{1,\ldots,H\},
\end{equation}
and define the corresponding mean token-level KL divergence as
\begin{equation}
    \bar{D}_{\mathrm{KL}}^{d_{\tilde{\pi}}^{\mathcal{T}_{i-1}}}(\tilde{\pi},\pi)
    =
    \mathbb{E}_{(t,x,y_{<t})\sim d_{\tilde{\pi}}^{\mathcal{T}_{i-1}}}
    \left[D_{\mathrm{KL}}\bigl(\tilde{\pi}(\cdot\mid x,y_{<t})\|\pi(\cdot\mid x,y_{<t})\bigr)\right].
\end{equation}
Then the old-task return degradation is bounded by
\begin{equation}
    \left|\eta_{i-1}(\tilde{\pi})-\eta_{i-1}(\pi)\right|
    \leq
    B_{i-1}\sqrt{\frac{H}{2}\bar{D}_{\mathrm{KL}}^{d_{\tilde{\pi}}^{\mathcal{T}_{i-1}}}(\tilde{\pi},\pi)}.
    \label{eq:app_kl_forgetting_bound}
\end{equation}
Moreover, the sequence-level KL and the normalized mean token-level KL satisfy
\begin{equation}
    \mathbb{E}_{x\sim\mathcal{D}_{i-1}}
    \left[D_{\mathrm{KL}}\bigl(P_{\tilde{\pi}}(\cdot\mid x)\|P_{\pi}(\cdot\mid x)\bigr)\right]
    =
    H\bar{D}_{\mathrm{KL}}^{d_{\tilde{\pi}}^{\mathcal{T}_{i-1}}}(\tilde{\pi},\pi).
    \label{eq:app_chain_rule_identity}
\end{equation}
\end{proposition}

\begin{proof}
We prove the result in three steps: first controlling terminal-reward difference by sequence-level KL, then expanding sequence-level KL by the autoregressive chain rule, and finally substituting the normalized mean token-level KL.

\paragraph{Step 1: terminal reward difference is controlled by sequence-level KL.}
Fix a prompt $x$ and abbreviate
\begin{equation}
    P=P_{\tilde{\pi}}(\cdot\mid x),\qquad
    Q=P_{\pi}(\cdot\mid x),\qquad
    R(y)=R_{i-1}(x,y).
\end{equation}
If $B_{i-1}=0$, then $R_{i-1}(x,y)$ is constant in $y$ for every $x$, so the left-hand side of Eq.~\ref{eq:app_kl_forgetting_bound} is zero and the proposition is immediate. Hence assume $B_{i-1}>0$.

For this fixed prompt, let
\begin{equation}
    R_x^{-}=\inf_y R_{i-1}(x,y),\qquad
    f_x(y)=\frac{R_{i-1}(x,y)-R_x^{-}}{B_{i-1}}.
\end{equation}
By the definition of $B_{i-1}$, $0\leq f_x(y)\leq1$ for all $y$. Since $\sum_y P(y)=\sum_y Q(y)=1$, subtracting the constant $R_x^{-}$ does not change the difference of expectations:
\begin{align}
    \mathbb{E}_{P}[R(y)]-\mathbb{E}_{Q}[R(y)]
    &=
    \sum_y (P(y)-Q(y))R(y) \\
    &=
    \sum_y (P(y)-Q(y))(R(y)-R_x^{-}) \\
    &=
    B_{i-1}\sum_y (P(y)-Q(y))f_x(y) \\
    &=
    B_{i-1}\left(\mathbb{E}_{P}[f_x]-\mathbb{E}_{Q}[f_x]\right).
\end{align}
Using the variational characterization of total variation,
\begin{equation}
    D_{\mathrm{TV}}(P,Q)
    =
    \sup_{0\leq f\leq1}\left|\mathbb{E}_{P}[f]-\mathbb{E}_{Q}[f]\right|,
\end{equation}
we obtain
\begin{equation}
    \left|\mathbb{E}_{P}[R(y)]-\mathbb{E}_{Q}[R(y)]\right|
    \leq B_{i-1}D_{\mathrm{TV}}(P,Q).
    \label{eq:app_tv_reward}
\end{equation}
Pinsker's inequality further gives
\begin{equation}
    D_{\mathrm{TV}}(P,Q)
    \leq
    \sqrt{\frac{1}{2}D_{\mathrm{KL}}(P\|Q)}.
\end{equation}
Combining the previous two inequalities and returning to the original notation yields, for every prompt $x$,
\begin{equation}
    \left|
    \mathbb{E}_{y\sim P_{\tilde{\pi}}(\cdot\mid x)}[R_{i-1}(x,y)]
    -
    \mathbb{E}_{y\sim P_{\pi}(\cdot\mid x)}[R_{i-1}(x,y)]
    \right|
    \leq
    B_{i-1}\sqrt{\frac{1}{2}D_{\mathrm{KL}}\bigl(P_{\tilde{\pi}}(\cdot\mid x)\|P_{\pi}(\cdot\mid x)\bigr)}.
    \label{eq:app_promptwise_kl_bound}
\end{equation}

Now expand the task-return difference:
\begin{align}
    \eta_{i-1}(\tilde{\pi})-\eta_{i-1}(\pi)
    &=
    \mathbb{E}_{x\sim\mathcal{D}_{i-1}}
    \left[
    \mathbb{E}_{y\sim P_{\tilde{\pi}}(\cdot\mid x)}[R_{i-1}(x,y)]
    -
    \mathbb{E}_{y\sim P_{\pi}(\cdot\mid x)}[R_{i-1}(x,y)]
    \right].
\end{align}
Taking absolute values, using $|\mathbb{E}[Z]|\leq\mathbb{E}[|Z|]$, and applying Eq.~\ref{eq:app_promptwise_kl_bound},
\begin{align}
    \left|\eta_{i-1}(\tilde{\pi})-\eta_{i-1}(\pi)\right|
    &\leq
    \mathbb{E}_{x\sim\mathcal{D}_{i-1}}
    \left[
    \left|
    \mathbb{E}_{P_{\tilde{\pi}}(\cdot\mid x)}[R_{i-1}(x,y)]
    -
    \mathbb{E}_{P_{\pi}(\cdot\mid x)}[R_{i-1}(x,y)]
    \right|
    \right] \\
    &\leq
    B_{i-1}\mathbb{E}_{x\sim\mathcal{D}_{i-1}}
    \left[
    \sqrt{\frac{1}{2}D_{\mathrm{KL}}\bigl(P_{\tilde{\pi}}(\cdot\mid x)\|P_{\pi}(\cdot\mid x)\bigr)}
    \right].
\end{align}
Because $u\mapsto\sqrt{u}$ is concave on $\mathbb{R}_{\geq0}$, Jensen's inequality implies
\begin{equation}
    \mathbb{E}_{x}\left[\sqrt{Z_x}\right]
    \leq
    \sqrt{\mathbb{E}_{x}[Z_x]}
    \qquad\text{for }Z_x\geq0.
\end{equation}
With $Z_x=\frac{1}{2}D_{\mathrm{KL}}\bigl(P_{\tilde{\pi}}(\cdot\mid x)\|P_{\pi}(\cdot\mid x)\bigr)$, this gives
\begin{equation}
    \left|\eta_{i-1}(\tilde{\pi})-\eta_{i-1}(\pi)\right|
    \leq
    B_{i-1}\sqrt{\frac{1}{2}
    \mathbb{E}_{x\sim\mathcal{D}_{i-1}}
    \left[D_{\mathrm{KL}}\bigl(P_{\tilde{\pi}}(\cdot\mid x)\|P_{\pi}(\cdot\mid x)\bigr)\right]}.
    \label{eq:app_sequence_kl_bound}
\end{equation}

\paragraph{Step 2: sequence-level KL equals the sum of token-level KLs.}
For a fixed prompt $x$, start from the definition of sequence-level KL:
\begin{align}
    D_{\mathrm{KL}}\bigl(P_{\tilde{\pi}}(\cdot\mid x)\|P_{\pi}(\cdot\mid x)\bigr)
    &=
    \sum_{y_{1:H}}P_{\tilde{\pi}}(y_{1:H}\mid x)
    \log\frac{P_{\tilde{\pi}}(y_{1:H}\mid x)}{P_{\pi}(y_{1:H}\mid x)}.
\end{align}
By autoregressive factorization,
\begin{equation}
    \frac{P_{\tilde{\pi}}(y_{1:H}\mid x)}{P_{\pi}(y_{1:H}\mid x)}
    =
    \prod_{t=1}^{H}
    \frac{\tilde{\pi}(y_t\mid x,y_{<t})}{\pi(y_t\mid x,y_{<t})},
\end{equation}
and therefore
\begin{equation}
    \log\frac{P_{\tilde{\pi}}(y_{1:H}\mid x)}{P_{\pi}(y_{1:H}\mid x)}
    =
    \sum_{t=1}^{H}
    \log\frac{\tilde{\pi}(y_t\mid x,y_{<t})}{\pi(y_t\mid x,y_{<t})}.
\end{equation}
Substituting this identity into the sequence KL and exchanging the finite sums gives
\begin{align}
    D_{\mathrm{KL}}\bigl(P_{\tilde{\pi}}(\cdot\mid x)\|P_{\pi}(\cdot\mid x)\bigr)
    &=
    \sum_{t=1}^{H}
    \sum_{y_{1:H}}P_{\tilde{\pi}}(y_{1:H}\mid x)
    \log\frac{\tilde{\pi}(y_t\mid x,y_{<t})}{\pi(y_t\mid x,y_{<t})}.
    \label{eq:app_chain_expand_full_sequence}
\end{align}
For a fixed $t$, decompose the sequence as $(y_{<t},y_t,y_{>t})$. Under $P_{\tilde{\pi}}$,
\begin{equation}
    P_{\tilde{\pi}}(y_{1:H}\mid x)
    =
    P_{\tilde{\pi}}(y_{<t}\mid x)
    \tilde{\pi}(y_t\mid x,y_{<t})
    P_{\tilde{\pi}}(y_{>t}\mid x,y_{\leq t}).
\end{equation}
The log-ratio term in Eq.~\ref{eq:app_chain_expand_full_sequence} depends only on $(x,y_{<t},y_t)$, so summing over future tokens gives
\begin{equation}
    \sum_{y_{>t}}P_{\tilde{\pi}}(y_{>t}\mid x,y_{\leq t})=1.
\end{equation}
Thus the $t$-th term becomes
\begin{align}
    &\sum_{y_{1:H}}P_{\tilde{\pi}}(y_{1:H}\mid x)
    \log\frac{\tilde{\pi}(y_t\mid x,y_{<t})}{\pi(y_t\mid x,y_{<t})} \\
    &\quad=
    \sum_{y_{<t}}P_{\tilde{\pi}}(y_{<t}\mid x)
    \sum_{y_t}\tilde{\pi}(y_t\mid x,y_{<t})
    \log\frac{\tilde{\pi}(y_t\mid x,y_{<t})}{\pi(y_t\mid x,y_{<t})} \\
    &\quad=
    \sum_{y_{<t}}P_{\tilde{\pi}}(y_{<t}\mid x)
    D_{\mathrm{KL}}\bigl(\tilde{\pi}(\cdot\mid x,y_{<t})\|\pi(\cdot\mid x,y_{<t})\bigr) \\
    &\quad=
    \mathbb{E}_{y_{<t}\sim P_{\tilde{\pi}}(\cdot\mid x)}
    \left[
    D_{\mathrm{KL}}\bigl(\tilde{\pi}(\cdot\mid x,y_{<t})\|\pi(\cdot\mid x,y_{<t})\bigr)
    \right].
\end{align}
Consequently,
\begin{equation}
    D_{\mathrm{KL}}\bigl(P_{\tilde{\pi}}(\cdot\mid x)\|P_{\pi}(\cdot\mid x)\bigr)
    =
    \sum_{t=1}^{H}
    \mathbb{E}_{y_{<t}\sim P_{\tilde{\pi}}(\cdot\mid x)}
    \left[
    D_{\mathrm{KL}}\bigl(\tilde{\pi}(\cdot\mid x,y_{<t})\|\pi(\cdot\mid x,y_{<t})\bigr)
    \right].
    \label{eq:app_autoregressive_chain_rule}
\end{equation}

\paragraph{Step 3: convert the chain rule into the normalized mean token-level KL.}
First verify that $d_{\tilde{\pi}}^{\mathcal{T}_{i-1}}$ is normalized:
\begin{align}
    \sum_{t=1}^{H}\sum_x\sum_{y_{<t}}
    d_{\tilde{\pi}}^{\mathcal{T}_{i-1}}(t,x,y_{<t})
    &=
    \frac{1}{H}\sum_{t=1}^{H}\sum_x\mathcal{D}_{i-1}(x)
    \sum_{y_{<t}}P_{\tilde{\pi}}(y_{<t}\mid x) \\
    &=
    \frac{1}{H}\sum_{t=1}^{H}\sum_x\mathcal{D}_{i-1}(x) \\
    &=1.
\end{align}
Therefore, by the definition of the mean token-level KL,
\begin{align}
    H\bar{D}_{\mathrm{KL}}^{d_{\tilde{\pi}}^{\mathcal{T}_{i-1}}}(\tilde{\pi},\pi)
    &=
    \sum_{t=1}^{H}
    \mathbb{E}_{x\sim\mathcal{D}_{i-1},\,y_{<t}\sim P_{\tilde{\pi}}(\cdot\mid x)}
    \left[
    D_{\mathrm{KL}}\bigl(\tilde{\pi}(\cdot\mid x,y_{<t})\|\pi(\cdot\mid x,y_{<t})\bigr)
    \right] \\
    &=
    \mathbb{E}_{x\sim\mathcal{D}_{i-1}}
    \left[D_{\mathrm{KL}}\bigl(P_{\tilde{\pi}}(\cdot\mid x)\|P_{\pi}(\cdot\mid x)\bigr)\right],
\end{align}
where the second equality uses Eq.~\ref{eq:app_autoregressive_chain_rule}. This proves Eq.~\ref{eq:app_chain_rule_identity}. Substituting Eq.~\ref{eq:app_chain_rule_identity} into Eq.~\ref{eq:app_sequence_kl_bound} gives
\begin{equation}
    \left|\eta_{i-1}(\tilde{\pi})-\eta_{i-1}(\pi)\right|
    \leq
    B_{i-1}\sqrt{\frac{H}{2}\bar{D}_{\mathrm{KL}}^{d_{\tilde{\pi}}^{\mathcal{T}_{i-1}}}(\tilde{\pi},\pi)},
\end{equation}
which is the desired result.
\end{proof}

\subsection{Parameter Movement as a Noisy Fisher Proxy}
\label{app:movement_fisher_proxy}

\begin{proposition}[Parameter movement as a noisy Fisher proxy]
\label{prop:app_movement_fisher_proxy}
Consider training on task $\mathcal{T}_{i-1}$ under a locally stationary approximation, so that the policy, normalized prefix-token distribution, Fisher diagonal, and reward-score moments are treated as fixed over the interval in which parameter movement is measured. Let $h_t=(x,y_{<t})$ and define the token score
\begin{equation}
    \phi_k(t,h_t,a)=\frac{\partial}{\partial\theta_k}\log\pi_\theta(a\mid h_t).
\end{equation}
The task Fisher diagonal is
\begin{equation}
    F_{kk}^{\mathcal{T}_{i-1}}
    =
    \mathbb{E}_{(t,h_t)\sim d_{\theta}^{\mathcal{T}_{i-1}}}
    \left[
    \mathbb{E}_{a\sim\pi_\theta(\cdot\mid h_t)}[\phi_k(t,h_t,a)^2]
    \right].
\end{equation}
For terminal-reward policy-gradient training, define the token-sampled gradient contribution
\begin{equation}
    z_k=A(x,y)\phi_k(t,h_t,y_t),
\end{equation}
where $A(x,y)$ is a terminal reward or baseline-centered terminal advantage, and $(t,h_t,y_t)$ is sampled from the normalized prefix-token distribution. Let
\begin{equation}
    \mu_k=\mathbb{E}[z_k],\qquad
    V_k=\operatorname{Var}(z_k),\qquad
    \sigma_A^2=\mathbb{E}[A^2],
\end{equation}
and, for $F_{kk}^{\mathcal{T}_{i-1}}>0$, define
\begin{equation}
    q_k=\frac{\mathbb{E}[A^2\phi_k^2]}{\sigma_A^2F_{kk}^{\mathcal{T}_{i-1}}},
    \qquad
    \rho_k=\frac{\mu_k}{\sigma_A\sqrt{F_{kk}^{\mathcal{T}_{i-1}}}}.
\end{equation}
Then
\begin{equation}
    V_k=\sigma_A^2(q_k-\rho_k^2)F_{kk}^{\mathcal{T}_{i-1}}.
    \label{eq:app_gradient_variance}
\end{equation}
Under SGD with learning rate $\eta$ for $N$ stochastic updates, if the parameter gradient samples have mean $\mu_k$ and martingale-difference noise with variance $V_k$, the expected squared movement satisfies
\begin{equation}
    \mathbb{E}[|\Delta\theta_k|^2]
    =
    \eta^2\sigma_A^2
    \left[N^2\rho_k^2+N(q_k-\rho_k^2)\right]
    F_{kk}^{\mathcal{T}_{i-1}}.
    \label{eq:app_movement_fisher_relation}
\end{equation}
Equivalently, with
\begin{equation}
    C_k=N^2\rho_k^2+N(q_k-\rho_k^2),
\end{equation}
we have
\begin{equation}
    \mathbb{E}[|\Delta\theta_k|^2]=\eta^2\sigma_A^2C_kF_{kk}^{\mathcal{T}_{i-1}}.
\end{equation}
Thus, ranking parameters by expected squared movement is consistent with ranking by $F_{kk}^{\mathcal{T}_{i-1}}$ if $C_k$ is constant across coordinates, and is pairwise consistent whenever, for $F_{kk}^{\mathcal{T}_{i-1}}>F_{\ell\ell}^{\mathcal{T}_{i-1}}$,
\begin{equation}
    \frac{C_{\ell}}{C_k}<\frac{F_{kk}^{\mathcal{T}_{i-1}}}{F_{\ell\ell}^{\mathcal{T}_{i-1}}}.
    \label{eq:app_pairwise_rank_condition}
\end{equation}
In the noise-dominated regime, this requires approximate homogeneity of $q_k-\rho_k^2$; in the bias-dominated regime, it requires approximate homogeneity of $|\rho_k|$.
\end{proposition}

\begin{proof}
We first relate the terminal-reward policy gradient to token-level scores. For the task return
\begin{equation}
    \eta_{i-1}(\theta)=\mathbb{E}_{x\sim\mathcal{D}_{i-1},\,y\sim P_{\theta}(\cdot\mid x)}[R_{i-1}(x,y)],
\end{equation}
the likelihood-ratio identity gives
\begin{equation}
    \nabla_{\theta_k}\eta_{i-1}(\theta)
    =
    \mathbb{E}\left[A(x,y)\sum_{t=1}^{H}\phi_k(t,h_t,y_t)\right],
\end{equation}
where subtracting any action-independent baseline from the terminal reward does not change the policy-gradient expectation. If $t$ is sampled uniformly from $\{1,\ldots,H\}$ after drawing a trajectory, then
\begin{equation}
    z_k=A(x,y)\phi_k(t,h_t,y_t)
\end{equation}
satisfies $\mathbb{E}[z_k]=H^{-1}\nabla_{\theta_k}\eta_{i-1}(\theta)$. The factor $H$ is shared across coordinates and therefore does not affect coordinate ranking.

Next, by definition of $q_k$ and $\rho_k$,
\begin{equation}
    \mathbb{E}[z_k^2]=\mathbb{E}[A^2\phi_k^2]=\sigma_A^2q_kF_{kk}^{\mathcal{T}_{i-1}},
    \qquad
    \mu_k^2=\sigma_A^2\rho_k^2F_{kk}^{\mathcal{T}_{i-1}}.
\end{equation}
Therefore,
\begin{equation}
    V_k=\operatorname{Var}(z_k)=\mathbb{E}[z_k^2]-\mu_k^2
    =\sigma_A^2(q_k-\rho_k^2)F_{kk}^{\mathcal{T}_{i-1}},
\end{equation}
which proves Eq.~\ref{eq:app_gradient_variance}.

Now consider $N$ locally stationary SGD updates in coordinate $k$:
\begin{equation}
    \theta_{r+1,k}=\theta_{r,k}-\eta z_{r,k},
    \qquad
    z_{r,k}=\mu_k+\xi_{r,k},
\end{equation}
where $\mathbb{E}[\xi_{r,k}\mid\mathcal{F}_{r-1}]=0$ and $\mathbb{E}[\xi_{r,k}^2\mid\mathcal{F}_{r-1}]=V_k$. The total movement is
\begin{equation}
    \Delta\theta_k=\theta_{N,k}-\theta_{0,k}
    =-
    \eta\sum_{r=1}^{N}z_{r,k}
    =-
    \eta(N\mu_k+M_{N,k}),
    \qquad
    M_{N,k}=\sum_{r=1}^{N}\xi_{r,k}.
\end{equation}
Using the martingale-difference property,
\begin{equation}
    \mathbb{E}[M_{N,k}]=0,
    \qquad
    \mathbb{E}[M_{N,k}^2]=\sum_{r=1}^{N}\mathbb{E}[\xi_{r,k}^2]=NV_k,
\end{equation}
where cross terms vanish because $\mathbb{E}[\xi_{r,k}\xi_{s,k}]=0$ for $r\neq s$. Hence
\begin{align}
    \mathbb{E}[|\Delta\theta_k|^2]
    &=
    \eta^2\mathbb{E}\left[(N\mu_k+M_{N,k})^2\right] \\
    &=
    \eta^2N^2\mu_k^2+
    2\eta^2N\mu_k\mathbb{E}[M_{N,k}]
    +
    \eta^2\mathbb{E}[M_{N,k}^2] \\
    &=
    \eta^2N^2\mu_k^2+
    \eta^2NV_k.
\end{align}
Substituting the expressions for $\mu_k^2$ and $V_k$ proves Eq.~\ref{eq:app_movement_fisher_relation}.

Finally, for two coordinates $k$ and $\ell$,
\begin{equation}
    \mathbb{E}[|\Delta\theta_k|^2]>
    \mathbb{E}[|\Delta\theta_{\ell}|^2]
    \quad\Longleftrightarrow\quad
    C_kF_{kk}^{\mathcal{T}_{i-1}}>C_{\ell}F_{\ell\ell}^{\mathcal{T}_{i-1}}.
\end{equation}
If $C_k$ is constant over coordinates, expected movement and Fisher diagonal induce the same ranking. More generally, when $F_{kk}^{\mathcal{T}_{i-1}}>F_{\ell\ell}^{\mathcal{T}_{i-1}}$, the Fisher ranking is preserved for the pair $(k,\ell)$ exactly under Eq.~\ref{eq:app_pairwise_rank_condition}.

The coefficient $C_k=N^2\rho_k^2+N(q_k-\rho_k^2)$ separates two regimes. If $N^2\rho_k^2\ll N(q_k-\rho_k^2)$, then $C_k\approx N(q_k-\rho_k^2)$, so rank consistency requires the noise coefficient $q_k-\rho_k^2$ to be sufficiently homogeneous or at least not large enough to violate Eq.~\ref{eq:app_pairwise_rank_condition}. If $N^2\rho_k^2\gg N(q_k-\rho_k^2)$, then $C_k\approx N^2\rho_k^2$, so rank consistency requires the normalized reward-score alignment magnitude $|\rho_k|$ to be sufficiently homogeneous. Since
\begin{equation}
    |\rho_k|=\frac{|\mathbb{E}[A\phi_k]|}{\sqrt{\mathbb{E}[A^2]}\sqrt{\mathbb{E}[\phi_k^2]}}\leq1
\end{equation}
by Cauchy-Schwarz, $\rho_k$ measures the normalized alignment between terminal advantage and the token score of coordinate $k$.

The realized movement $|\Delta\theta_k|$ is therefore a noisy proxy for $F_{kk}^{\mathcal{T}_{i-1}}$: it is governed by the Fisher diagonal in expectation, but its ranking is mediated by $C_k$ and by stochastic-gradient fluctuations. Under bounded or sub-Gaussian martingale noise, standard concentration inequalities imply that the realized ranking agrees with the expected ranking whenever the pairwise expected movement gap is larger than the corresponding concentration error.
\end{proof}

\makeatletter
\@ifundefined{cpocmain}{\newtheorem{cpocmain}{Theorem}[section]}{}
\@ifundefined{cpoccor}{\newtheorem{cpoccor}[cpocmain]{Corollary}}{}
\@ifundefined{cpocprop}{\newtheorem{cpocprop}[cpocmain]{Proposition}}{}
\makeatother

\section{Theoretical Forgetting Bound for CPO}
\label{cpoc:sec}

The derivation of CPO in section~\ref{sec:method} replaces the prior-task behavioral KL under the oracle model with a replay-free surrogate through a sequence of approximations. This appendix formalizes this approximation chain and quantifies the resulting error. We analyze the transition from $\mathcal T_{i-1}$ to $\mathcal T_i$; the same arguments extend to any earlier task $\mathcal T_j$ because the cumulative mask can only increase its covered Fisher sensitivity over time. The analysis allows heterogeneous coefficients $C_k$ and explicitly accounts for deviations between realized and expected parameter movement.

We first relate movement-based selection to Fisher sensitivity, and then use the resulting coverage bound to control the full Fisher quadratic form, behavioral KL, and return degradation. Finally, we characterize the gap to the Fisher oracle in terms of coefficient heterogeneity and movement-estimation error, with a high-probability bound for the latter under sub-Gaussian noise.

\subsection{Notation and Setup}
\label{cpoc:setup}
For completeness, we first summarize the notation used in the analysis and
formalize the movement-based parameter selection in CPO.

Fix a transition from $\mathcal T_{i-1}$ to $\mathcal T_i$ with $i\geq2$.
As in the main text, let
$\widetilde\theta=\theta^{(i-1)}$,
$\widetilde\pi=\pi_{\widetilde\theta}$, and let $\theta$ denote the parameters
during subsequent adaptation.
We distinguish the historical movement used for selection from the subsequent
displacement being regularized:
\begin{equation}
\label{cpoc:displacements}
\Delta\theta_k=\widetilde\theta_k-\theta^{(i-2)}_k,
\qquad
s^{(i-1)}_k=|\Delta\theta_k|,
\qquad
\delta=\theta-\widetilde\theta.
\end{equation}
For brevity, throughout this appendix write
\begin{equation}
\label{cpoc:abbreviations}
F=F^{\mathcal T_{i-1}}_{\widetilde\theta}\succeq0,
\qquad
F_{kk}=F^{\mathcal T_{i-1}}_{kk},
\qquad
m=m^{(i-1)}\in\{0,1\}^d.
\end{equation}
Thus $m$ always denotes the binary cumulative mask; it is not a movement score.
All coordinates refer to the trainable parameterization, with the remaining
parameters held fixed.

\paragraph{Movement scores under the local model.}
Under the locally stationary SGD model of Proposition~\ref{prop:app_movement_fisher_proxy}, with
$\eta>0$, $N\geq1$, and $0<\sigma_A^2=\mathbb E[A^2]<\infty$, define
\begin{equation}
\label{cpoc:scores}
\widehat u_k
:=\frac{(s^{(i-1)}_k)^2}{\eta^2\sigma_A^2},
\qquad
u_k:=\mathbb E[\widehat u_k]=C_kF_{kk},
\qquad
C_k=N^2\rho_k^2+N(q_k-\rho_k^2).
\end{equation}
Here $q_k$, $\rho_k$, $\mu_k$, and $V_k$ retain their definitions in
Appendix~\ref{app:movement_fisher_proxy}, and $C_k\geq0$.
All score and gradient second moments are assumed finite.
The expectation treats the local policy, prefix distribution, and score moments
as fixed, as in that approximation.
For $F_{kk}=0$, the score vanishes almost surely under this model; we set
$C_k=u_k=0$ by convention.
Ranking by $\widehat u_k$ is exactly equivalent to ranking by $s^{(i-1)}_k$,
since squaring and a shared positive rescaling preserve this realized ranking.
The deviation of the realized score from its expectation is quantified
explicitly in section~\ref{cpoc:noise}.

\paragraph{The selection rule and its Fisher comparator.}
Keep the tensor partition $\mathcal G=\{G_1,\ldots,G_M\}$ from the main text.
For each tensor $G_g$, let $b_g$ be the number of coordinates returned by its
$\operatorname{TopP}_p$ operation, including its fixed rounding rule.
Write $S_{s,g}$ for these coordinates at the end of $\mathcal T_{i-1}$, and define
\begin{equation}
\label{cpoc:sets}
\begin{aligned}
S_{s,g}
&\in\mathop{\operatorname{arg\,max}}_{
             S\subseteq G_g,\ |S|=b_g}
          \sum_{k\in S}\widehat u_k,
&
S_{F,g}
&\in\mathop{\operatorname{arg\,max}}_{
             S\subseteq G_g,\ |S|=b_g}
          \sum_{k\in S}F_{kk}.
\end{aligned}
\end{equation}
Ties may be resolved by any fixed rule.
$S_{F,g}$ serves as the Fisher-optimal comparator under the same per-tensor
selection budget as CPO.
By the cumulative update,
$\bigcup_g S_{s,g}\subseteq\mathcal S^{(i-1)}=\{k:m_k=1\}$.

For each group, define its oracle sensitivity and largest coefficient by
\begin{equation}
\label{cpoc:groupquantities}
\Phi_g^\star=\sum_{k\in S_{F,g}}F_{kk},
\qquad
C_{\max,g}=\max_{k\in G_g}C_k.
\end{equation}
When $\Phi_g^\star>0$ and $0<C_{\max,g}<\infty$, let
\begin{equation}
\label{cpoc:gamma}
\overline C_{F,g}
=\frac{\sum_{k\in S_{F,g}}C_kF_{kk}}{\Phi_g^\star},
\qquad
\gamma_g=\frac{\overline C_{F,g}}{C_{\max,g}}\in[0,1].
\end{equation}
The average is weighted by Fisher sensitivity, allowing heterogeneous
coordinate-wise coefficients $C_k$.

For nonnegative tolerances $\varepsilon_1,\ldots,\varepsilon_M$, consider the event
\begin{equation}
\label{cpoc:event}
\mathcal E(\varepsilon)
=\left\{
 \max_{k\in G_g}|\widehat u_k-u_k|\leq\varepsilon_g
 \text{ for every }g
 \right\}.
\end{equation}
Define the certified covered sensitivity
\begin{equation}
\label{cpoc:lambda}
\Lambda=\sum_{g=1}^M\Lambda_g,
\qquad
\Lambda_g=
\left[\gamma_g\Phi_g^\star
      -\frac{2b_g\varepsilon_g}{C_{\max,g}}\right]_+,
\qquad [x]_+:=\max\{x,0\}.
\end{equation}
The expression for $\Lambda_g$ applies to the groups with
$\Phi_g^\star>0$ and $0<C_{\max,g}<\infty$.
Set $\Lambda_g=0$ whenever $b_g=0$, $\Phi_g^\star=0$, or
$C_{\max,g}=0$, such a group contributes no positive lower bound.

\subsection{Main theorem: from the mask to a forgetting bound}
\label{cpoc:mainsec}

The total diagonal sensitivity covered by a mask and its uncovered remainder are
\begin{equation}
\label{cpoc:coverage}
\mathcal I_F(m):=\sum_{k=1}^d m_kF_{kk},
\qquad
\tau_F(m):=\sum_{k=1}^d(1-m_k)F_{kk}
            =\operatorname{tr}(F)-\mathcal I_F(m).
\end{equation}

Let
\begin{equation}
\label{cpoc:regularizer}
\Omega_{i-1}(\theta)
:=\frac{\|m\odot\delta\|_1}{\|m\|_0},
\qquad
a_F(m):=\max_{k:m_k=1}\sqrt{F_{kk}}.
\end{equation}
For an empty mask, set both quantities to zero.
$\Omega_{i-1}$ is exactly the normalized masked movement penalty in Eq.~\ref{eq:l1_loss_vector},
before multiplication by $\lambda$.

\begin{cpocmain}[Movement-selected sensitivity and local forgetting]
\label{cpoc:main}
Assume the setting of section~\ref{cpoc:setup} and consider a realization in
$\mathcal E(\varepsilon)$.
Let $m$ be any cumulative mask containing $\bigcup_g S_{s,g}$.
Then
\begin{equation}
\label{cpoc:maincoverage}
\mathcal I_F(m)\geq\Lambda,
\qquad
\tau_F(m)\leq\operatorname{tr}(F)-\Lambda.
\end{equation}
Define
\begin{equation}
\label{cpoc:driftfactor}
D_m(\theta)
:=a_F(m)\|m\|_0\,\Omega_{i-1}(\theta)
 +\sqrt{\operatorname{tr}(F)-\Lambda}\,
       \|(\mathbf 1-m)\odot\delta\|_2.
\end{equation}
For the full Fisher matrix, including arbitrary off-diagonal interactions,
\begin{equation}
\label{cpoc:mainfisher}
\delta^\top F\delta\leq D_m(\theta)^2.
\end{equation}
Suppose the prior-task behavioral KL has the expansion from section~\ref{subsec:fisher_relaxation},
\begin{equation}
\label{cpoc:taylorexact}
K_{i-1}(\theta)
=\frac12\delta^\top F\delta+R^{(3)}_{i-1}(\theta).
\end{equation}
Then
\begin{equation}
\label{cpoc:mainkl}
K_{i-1}(\theta)
\leq\frac12D_m(\theta)^2+|R^{(3)}_{i-1}(\theta)|.
\end{equation}
For bounded old-task reward range $B_{i-1}$ and generation horizon $H$,
\begin{equation}
\label{cpoc:mainforgetting}
\left|\eta_{i-1}(\widetilde\pi)-\eta_{i-1}(\pi_\theta)\right|
\leq B_{i-1}
\sqrt{\frac H4 D_m(\theta)^2
      +\frac H2|R^{(3)}_{i-1}(\theta)|}.
\end{equation}
If additionally $\nabla^2K_{i-1}$ is $L$-Lipschitz in operator norm on a
neighborhood containing the segment from $\widetilde\theta$ to $\theta$,
then $|R^{(3)}_{i-1}(\theta)|\leq L\|\delta\|_2^3/6$, and hence
\begin{equation}
\label{cpoc:smoothforgetting}
\left|\eta_{i-1}(\widetilde\pi)-\eta_{i-1}(\pi_\theta)\right|
\leq B_{i-1}
\sqrt{\frac H4D_m(\theta)^2+\frac{HL}{12}\|\delta\|_2^3}.
\end{equation}
\end{cpocmain}

\begin{proof}
We prove the result in three steps: first showing that the selection rule retains aggregate Fisher sensitivity, then controlling the full Fisher quadratic form via a covered/uncovered decomposition, and finally translating this bound into forgetting through KL continuity and smoothness.

\textbf{Step 1: movement selection retains aggregate Fisher sensitivity.}
Fix a group $G_g$ with $C_{\max,g}>0$.
On $\mathcal E(\varepsilon)$, since $S_{s,g}$ maximizes the selection criterion,
\begin{equation}
\label{cpoc:selectionproof}
\begin{aligned}
\sum_{k\in S_{s,g}}C_kF_{kk}
&=\sum_{k\in S_{s,g}}u_k\\
&\geq\sum_{k\in S_{s,g}}\widehat u_k-b_g\varepsilon_g\\
&\geq\sum_{k\in S_{F,g}}\widehat u_k-b_g\varepsilon_g\\
&\geq\sum_{k\in S_{F,g}}u_k-2b_g\varepsilon_g\\
&=\overline C_{F,g}\Phi_g^\star-2b_g\varepsilon_g.
\end{aligned}
\end{equation}
Since $C_k\leq C_{\max,g}$ and $F_{kk}\geq0$, we have
\[
\sum_{k\in S_{s,g}}C_kF_{kk}
\leq C_{\max,g}\sum_{k\in S_{s,g}}F_{kk}.
\]
Combining this with \ref{cpoc:selectionproof} gives
\begin{equation}
\label{cpoc:onegroup}
\sum_{k\in S_{s,g}}F_{kk}
\geq
\left[
\frac{\overline C_{F,g}\Phi_g^\star-2b_g\varepsilon_g}{C_{\max,g}}
\right]_+
=\Lambda_g.
\end{equation}
For groups with $C_{\max,g}=0$, we set $\Lambda_g=0$; then the same lower bound holds immediately because $F_{kk}\geq0$.

The groups are disjoint, and the cumulative mask contains all newly selected coordinates, so
\begin{equation}
\label{cpoc:unionproof}
\mathcal I_F(m)
\geq
\sum_{g=1}^M\sum_{k\in S_{s,g}}F_{kk}
\geq
\sum_{g=1}^M\Lambda_g
=\Lambda.
\end{equation}
Subtracting this lower bound from $\operatorname{tr}(F)$ proves Eq.~\ref{cpoc:maincoverage}.

\textbf{Step 2: covered sensitivity and masked movement control the full quadratic form.}
Let $S=\{k:m_k=1\}$ and let $e_k$ be the $k$-th coordinate vector.
Positive semidefiniteness gives
$\|F^{1/2}e_k\|_2=\sqrt{F_{kk}}$.
Consequently, the triangle and Cauchy--Schwarz inequalities imply
\begin{equation}
\label{cpoc:psdproof}
\begin{aligned}
\sqrt{\delta^\top F\delta}
&=\left\|\sum_{k=1}^d\delta_kF^{1/2}e_k\right\|_2\\
&\leq\sum_{k\in S}|\delta_k|\sqrt{F_{kk}}
     +\sum_{k\notin S}|\delta_k|\sqrt{F_{kk}}\\
&\leq a_F(m)\|m\odot\delta\|_1
     +\sqrt{\tau_F(m)}\,\|(\mathbf 1-m)\odot\delta\|_2\\
&\leq a_F(m)\|m\|_0\,\Omega_{i-1}(\theta)
     +\sqrt{\operatorname{tr}(F)-\Lambda}\,
          \|(\mathbf 1-m)\odot\delta\|_2.
\end{aligned}
\end{equation}
Squaring proves Eq.~\ref{cpoc:mainfisher}.
In particular, off-diagonal Fisher interactions are included through the
positive-semidefinite bound rather than discarded.

\textbf{Step 3: pass from the quadratic form to behavioral KL and returns.}
Substituting Eq.~\ref{cpoc:mainfisher} into
Eq.~\ref{cpoc:taylorexact} and bounding the remainder by its absolute value
proves Eq.~\ref{cpoc:mainkl}.
Proposition~\ref{prop:app_kl_forgetting_bound} gives
\begin{equation}
\label{cpoc:pinskerstep}
\left|\eta_{i-1}(\widetilde\pi)-\eta_{i-1}(\pi_\theta)\right|
\leq B_{i-1}\sqrt{\frac H2 K_{i-1}(\theta)},
\end{equation}
which proves Eq.~\ref{cpoc:mainforgetting}.
For the smoothness specialization, using
$K_{i-1}(\widetilde\theta)=0$ and
$\nabla K_{i-1}(\widetilde\theta)=0$, the integral remainder satisfies
\begin{equation}
\label{cpoc:remainderproof}
\begin{aligned}
|R^{(3)}_{i-1}(\theta)|
&\leq\int_0^1(1-t)\|\delta\|_2^2
   \left\|\nabla^2K_{i-1}(\widetilde\theta+t\delta)-F\right\|_{\mathrm{op}}
   \,dt\\
&\leq L\|\delta\|_2^3\int_0^1t(1-t)\,dt
 =\frac L6\|\delta\|_2^3.
\end{aligned}
\end{equation}
Substitution yields Eq.~\ref{cpoc:smoothforgetting}.
\end{proof}

\paragraph{Role of the CPO regularizer.}
The bound separates the roles of selection and regularization.
The selection rule determines $\Lambda$, and thus controls the Fisher
sensitivity that may remain outside the mask. The masked $L_1$ penalty
controls the parameter displacement on the selected coordinates.
Together, these two terms bound the full Fisher quadratic form.
In particular, for an iterate satisfying
$\Omega_{i-1}(\theta)\leq r_1$ and
$\|(\mathbf 1-m)\odot\delta\|_2\leq r_2$, we obtain
\begin{equation}
\label{cpoc:budget}
D_m(\theta)
\leq a_F(m)\|m\|_0r_1
      +\sqrt{\operatorname{tr}(F)-\Lambda}\,r_2.
\end{equation}
The factor $\|m\|_0$ follows from the normalized masked $L_1$
regularizer in Eq.~\ref{eq:l1_loss_vector}.
\subsection{Controlling the Uncovered Fisher Sensitivity}
\label{cpoc:interpretation}

Theorem~\ref{cpoc:main} bounds the contribution of the uncovered coordinates through the factor
$\sqrt{\operatorname{tr}(F)-\Lambda}$.
We compare this factor with the Fisher oracle under the same per-tensor selection budgets.

Define the oracle residual
\[
\tau^\star:=\operatorname{tr}(F)-\sum_g\Phi_g^\star.
\]
For each group with $\Phi_g^\star>0$ and $C_{\max,g}>0$, let
$e_g:=2b_g\varepsilon_g/C_{\max,g}$; set $\gamma_g=e_g=0$ for the remaining groups,
consistently with $\Lambda_g=0$.
On $\mathcal E(\varepsilon)$, the inequality
$\Lambda_g=[\gamma_g\Phi_g^\star-e_g]_+
\geq\gamma_g\Phi_g^\star-e_g$ gives
\begin{equation}
\label{cpoc:residualdecomposition}
\begin{aligned}
\tau_F(m)
&\leq \operatorname{tr}(F)-\Lambda\\
&\leq
\underbrace{\tau^\star}_{\text{oracle tail}}
+\underbrace{\sum_g(1-\gamma_g)\Phi_g^\star}_{\text{coefficient loss}}
+\underbrace{\sum_g e_g}_{\text{movement error}}.
\end{aligned}
\end{equation}
Selection approaches the oracle residual when the latter two terms are small. The following sufficient condition controls the coefficient loss under heterogeneous $C_k$.

\begin{cpoccor}[Uncovered sensitivity under partial excitation]
\label{cpoc:excitation}
For each group with $\Phi_g^\star>0$ and $C_{\max,g}>0$, suppose there exist
$B_g\subseteq S_{F,g}$ and $\alpha_g,\beta_g\in[0,1]$ such that
\begin{equation}
\label{cpoc:excitationcondition}
\sum_{k\in B_g}F_{kk}\geq\alpha_g\Phi_g^\star,
\qquad
C_k\geq\beta_gC_{\max,g}
\quad\text{for every }k\in B_g.
\end{equation}
Set $\alpha_g=\beta_g=0$ for the remaining groups.
Then $\gamma_g\geq\alpha_g\beta_g$, and on $\mathcal E(\varepsilon)$,
\begin{equation}
\label{cpoc:excitationbound}
\sqrt{\operatorname{tr}(F)-\Lambda}
\leq
\left[
\tau^\star
+\sum_g(1-\alpha_g\beta_g)\Phi_g^\star
+\sum_g e_g
\right]^{1/2}.
\end{equation}
\end{cpoccor}

\begin{proof}
For a group with $\Phi_g^\star>0$ and $C_{\max,g}>0$, nonnegativity and
Eq.~\ref{cpoc:excitationcondition} imply
\begin{equation}
\gamma_g
=\frac{\sum_{k\in S_{F,g}}C_kF_{kk}}
       {C_{\max,g}\Phi_g^\star}
\geq
\beta_g\frac{\sum_{k\in B_g}F_{kk}}{\Phi_g^\star}
\geq\alpha_g\beta_g.
\end{equation}
The claim follows from Eq.~\ref{cpoc:residualdecomposition} and
monotonicity of the square root.
\end{proof}

The condition is aggregate. It requires a sufficient fraction of the Fisher sensitivity to exhibit appreciable movement coupling. It allows unrestricted coefficient variation outside $B_g$. This requirement is substantially weaker than coordinate-wise coefficient uniformity.

The coefficient multiplying the uncovered coordinates is small when the selection budgets cover most Fisher sensitivity, the Fisher-weighted coefficient loss is small, and the realized movement errors are controlled.
Substituting Eq.~\ref{cpoc:excitationbound} into $D_m(\theta)$ gives the corresponding KL and return-degradation bounds in Theorem~\ref{cpoc:main}, with the unprotected displacement and Taylor remainder left explicit.
\subsection{Controlling the Realized-Movement Error}
\label{cpoc:noise}

Section~\ref{cpoc:interpretation} isolates the effect of realized-movement
noise through the error term
$e_g=2b_g\varepsilon_g/C_{\max,g}$.
We now bound $\varepsilon_g$ with high probability under an additional
sub-Gaussian assumption on the stochastic-gradient noise.

\begin{cpocprop}[A simultaneous movement-error bound]
\label{cpoc:concentration}
Let the noise in Appendix~\ref{app:movement_fisher_proxy} satisfy
$z_{n,k}=\mu_k+\xi_{n,k}$,
$\mathbb E[\xi_{n,k}\mid\mathcal F_{n-1}]=0$, and
$\mathbb E[\xi_{n,k}^2\mid\mathcal F_{n-1}]=V_k$.
Assume additionally that, for deterministic $v_k\geq0$,
\begin{equation}
\label{cpoc:subgaussian}
\mathbb E[\exp(t\xi_{n,k})\mid\mathcal F_{n-1}]
\leq\exp(t^2v_k^2/2)
\quad\text{for all }t\in\mathbb R.
\end{equation}
For $\zeta\in(0,1)$, define
\begin{equation}
\label{cpoc:stochasticerror}
\begin{aligned}
a_k(\zeta)&=v_k\sqrt{2N\log(2d/\zeta)},\\
\varepsilon_g(\zeta)
&=\max_{k\in G_g}
  \frac{2N|\mu_k|a_k(\zeta)+a_k(\zeta)^2+NV_k}{\sigma_A^2}.
\end{aligned}
\end{equation}
Then $\mathbb P(\mathcal E(\varepsilon(\zeta)))\geq1-\zeta$.
Consequently, the conclusions of Theorem~\ref{cpoc:main} hold with at least this
probability, with respect to the historical movement noise, using these
tolerances.
\end{cpocprop}

\begin{proof}
Let $M_{N,k}=\sum_{n=1}^N\xi_{n,k}$.
Iterated conditioning in Eq.~\ref{cpoc:subgaussian} gives
$\mathbb E[\exp(tM_{N,k})]\leq\exp(Nt^2v_k^2/2)$.
Applying the exponential Markov inequality to both tails and optimizing $t$
yields
\begin{equation}
\mathbb P\bigl(|M_{N,k}|>a_k(\zeta)\bigr)\leq\zeta/d.
\end{equation}
For $v_k=0$, $M_{N,k}=0$ almost surely and the conclusion is immediate.
A union bound makes $|M_{N,k}|\leq a_k(\zeta)$ simultaneous for all coordinates
with probability at least $1-\zeta$.
On this event, the SGD identity
$\Delta\theta_k=-\eta(N\mu_k+M_{N,k})$ implies
\begin{equation}
\label{cpoc:errorproof}
\begin{aligned}
|\widehat u_k-u_k|
&=\frac{|2N\mu_kM_{N,k}+M_{N,k}^2-NV_k|}{\sigma_A^2}\\
&\leq\frac{2N|\mu_k|a_k(\zeta)+a_k(\zeta)^2+NV_k}{\sigma_A^2}.
\end{aligned}
\end{equation}
Taking the maximum within each group proves the result.
\end{proof}

\paragraph{Remark.}
The above results require neither uniform $C_k$, exact Fisher ranking recovery,
nor a diagonal Fisher matrix. Their validity is instead governed by three
quantities made explicit in the bounds: the aggregate Fisher sensitivity retained
by the mask, the deviation between realized and expected movement, and the
subsequent parameter drift. The sub-Gaussian assumption in
Proposition~\ref{cpoc:concentration} is only used to obtain a high-probability
control of the realized-movement error and is not required for the deterministic
statements in Theorem~\ref{cpoc:main}.

\section{Details of MRCL}
\label{sec:details_mrcl}

This section describes the construction procedure for each task in MRCL.

\textbf{MedBookVQA.} This task evaluates medical visual understanding in difficult clinical scenarios. We select the Symptom Recognition and Surgery Operation subsets of \citet{yip2025medbookvqa}. After filtering out a small number of samples with parsing failures, we randomly split the remaining data into training and testing sets using a 7:3 ratio.

\textbf{Navigation.} The Navigation task is derived from the ThinkMorph \citep{gu2026thinkmorph} dataset. We implement a strict manual filtering process to ensure that each sample has a strictly unique shortest path, making the evaluation unambiguous. We then divide the data by the required number of navigation steps: samples requiring 2 to 10 steps are split into training and testing sets at a 7:3 ratio, and all samples requiring 11 to 15 steps are assigned to the test set.

\textbf{We-Math2.0.} This task evaluates complex multimodal mathematical reasoning and is processed into two configurations. For We-Math2.0-standard \citep{qiao2026wemath}, we extract problems from Plane Geometry, Analytic Geometry, and Solid Geometry. To ensure broad coverage, we sample one question per fundamental principle attribute to construct the test set, leaving the rest for training. For We-Math2.0-pro \citep{qiao2026wemath}, we separate the dataset based on difficulty attributes. Problems with foundational difficulties (base x, y, z) are designated as the training set, and those with extended combinatorial difficulties (xy, xz, yz, xyz) form the test set. Additionally, we conduct a comprehensive calibration of the ground-truth labels so that mathematical expressions strictly conform to standard LaTeX formatting.

\textbf{Puzzle.} This task is designed to test visual-spatial reasoning sourced from the ThinkMorph \citep{gu2026thinkmorph} datasets. We use the Jigsaw\_Assembly subset entirely for training and the VisPuzzle subset exclusively for testing.

\textbf{FinMME.} The financial chart understanding task requires precise visual extraction and numerical reasoning \citep{luo2025finmme}. We clean the dataset by removing images that fail to parse, with half of the questions being multi-answer choice where more than one option may be correct and the other half being numerical questions. Recognizing the inherent visual discrepancies when models estimate values directly from charts, we introduce a reasonable tolerance threshold for the numerical ground truths to ensure fair evaluation.

\section{Implementation Details}
\label{sec:impl}
We implement all baseline methods and CPO within the \texttt{TRL} framework following their original formulations and corresponding TRL implementations. We conduct experiments with Qwen3-VL-2B, Qwen3-VL-4B, and Qwen3-VL-8B \citep{bai2025qwen3}, and all reported SFT and RL runs optimize the full model. Our experimental setting uses only the original task ground-truth answers and excludes externally generated or manually annotated reasoning traces to avoid additional supervision from CoT annotations. Accordingly, SFT uses direct question-answer pairs with System 1 prompting, while all RL methods use the same CoT prompting template during rollout generation without supervised CoT targets. For all RL methods, we use AdamW with a learning rate of $5\times10^{-6}$, weight decay of $0.1$, a warmup ratio of $0.1$, and cosine learning rate decay. Training uses 8 rollouts per prompt, a sampling temperature of $0.8$, and a maximum completion length of $2048$. MedBookVQA is trained for 300 steps across all model scales. The remaining MRCL tasks are trained for 600 steps on Qwen3-VL-2B and Qwen3-VL-4B. Due to computational resource constraints, we reduce the training steps to 400 for Qwen3-VL-8B.
For GRPO \citep{shao2024deepseekmath}, we adopt the DAPO improvements \citep{yu2025dapo} with $\epsilon_{\text{low}}=0.2$ and $\epsilon_{\text{high}}=0.27$ under token-level policy gradient loss. GRPO* additionally uses a KL penalty coefficient of $\beta=0.04$. Dr.GRPO \citep{liu2025understanding} normalizes the policy loss by the fixed maximum completion length, disables reward standard deviation scaling, and uses symmetric clipping with $\epsilon_{\text{low}}=\epsilon_{\text{high}}=0.2$. For GSPO \citep{zheng2025group}, we set $\epsilon_{\text{low}}=3\times10^{-4}$ and $\epsilon_{\text{high}}=4\times10^{-4}$. CPO builds on GSPO and applies sparse regularization with $\lambda=100$. After each task, parameter movement is ranked independently within every trainable parameter tensor, and the top $p$ of coordinates in each tensor are added to the cumulative protected set. We set $p = 10\%$ for Qwen3-VL-2B and $p = 2.5\%$ for Qwen3-VL-4B and Qwen3-VL-8B. All experiments use NVIDIA H100 GPUs. Full SFT uses 2 GPUs for the 2B and 4B models and 4 GPUs for the 8B model, while all RL experiments use 8 GPUs.

During inference, we use temperature $0.0$ and repetition penalty $1.05$ with baseline $\text{Zero-shot}^2$. The maximum generation length is $2048$ for the five MRCL tasks and $16384$ for external capability benchmarks. 

We adopt distinct scoring rules tailored to the answer formats of different tasks. For \textit{MedBookVQA} and \textit{Puzzle}, a prediction receives a score of 1 for an exact match and 0 otherwise. For \textit{Navigation}, a prediction receives a score of 1 only if the entire generated trajectory perfectly matches the ground truth shortest path. For \textit{We-Math2.0}, answers containing multiple blanks separated by semicolons receive partial credit, and each correct blank contributes $1/k$ to the score where $k$ is the total number of blanks. For \textit{FinMME}, multiple-answer choice questions receive a score of 0.5 for incomplete correct selections and 0 if any incorrect option is chosen. For numerical questions in \textit{FinMME}, a prediction receives a score of 1 if it falls within a predefined tolerance threshold.

The evaluation of diverse multimodal benchmarks to measure retention of pretrained capabilities covers STEM benchmarks MMMU-Pro \citep{yue2025mmmu}, MathVerse \citep{zhang2024mathverse}, MathVision \citep{wang2024measuring}, and MathVista \citep{lu2024mathvista}; general VQA benchmarks RealworldQA \citep{realworldqa2024} and MMStar \citep{chen2024we}; the hallucination benchmark POPE \citep{li2023evaluating}; chart and document understanding benchmarks DocVQA \citep{mathew2021docvqa} and CharXiv \citep{wang2024charxiv}; and the grounding benchmark CountBenchQA \citep{paiss2023teaching}. We report F1 for POPE. For MMMU-Pro, we average accuracy over the standard 4-option, standard 10-option, and vision subsets. All other benchmarks use accuracy.

\subsection{CAI and CAD Metrics for Fig.~\ref{fig:intro}}

\label{sec:caicad}
To quantify the dynamics of plasticity and stability visualized in Fig. \ref{fig:intro}, we introduce two trajectory metrics corresponding to the axes: Cumulative mean Accuracy Improvement (CAI) and Cumulative mean Accuracy Delta (CAD). Let $T$ denote the total number of tasks, $a_{i,j}$ denote the accuracy on task $j$ after training on task $i$, and $z_j$ denote the zero-shot accuracy on task $j$.

The Cumulative mean Accuracy Improvement measures the average performance gain immediately after training on each task compared to its zero-shot baseline, accumulated from the first task through the current task $i$:
\begin{equation}
    \text{CAI}_i = \frac{1}{T} \sum_{j=1}^{i} (a_{j,j} - z_j)
\end{equation}

The Cumulative mean Accuracy Delta measures the accumulated performance reduction on previously learned tasks. In each sequential stage $k$ from 2 to $i$, we compute the average accuracy delta for all previous tasks $j < k$ after learning the new task $k$:
\begin{equation}
    \text{CAD}_i = \sum_{k=2}^{i} \frac{1}{k-1} \sum_{j=1}^{k-1} (a_{k,j} - a_{k-1,j})
\end{equation}
A negative $\text{CAD}_i$ value indicates cumulative forgetting, while a positive value suggests backward transfer that improves performance on previous tasks.

\subsection{Teacher-Forced Answer-Token NLL for Fig.~\ref{fig:nll_familiarity}}
\label{sec:details_nll}

To assess benchmark familiarity from a complementary perspective, we measure the conditional likelihood assigned by a frozen Qwen3-VL-4B-Instruct model to the canonical answer of each test example. For sample $i$, let $c_i$ denote the multimodal prompt and let $y_i=(y_{i,1},\ldots,y_{i,L_i})$ denote the tokenized canonical answer. We compute the teacher-forced answer-token NLL as
\begin{equation}
\operatorname{NLL}^{\mathrm{tok}}_i
=
-\frac{1}{L_i}
\sum_{t=1}^{L_i}
\log p_\theta
\left(
y_{i,t}
\mid
c_i,y_{i,<t}
\right).
\end{equation}
The model remains frozen throughout evaluation, with no rollout or parameter update. Token-level normalization reduces the direct influence of answer length and enables more meaningful comparisons across benchmarks.

All tasks are evaluated under fixed task-specific protocols designed to preserve their original answer semantics while minimizing artifacts introduced by answer formatting. Fig.~\ref{fig:nll_familiarity} summarizes the resulting sample-level NLL distributions and their task means. Lower NLL indicates that the model assigns higher conditional likelihood to the canonical answers, while higher NLL indicates lower conditional-answer familiarity.

\section{$L_1$ versus $L_2$ Movement Regularization}
\label{sec:suppl_exp}

\textbf{Effect of $L_1$ versus $L_2$ movement regularization.}
We further compare CPO with masked $L_1$ movement regularization against a masked $L_2$ variant on Qwen3-VL-4B. Fig.~\ref{fig:coefficientL1L2} shows that both regularization terms are active during training, as their losses consistently rise and decay across continual stages. Therefore, the performance gap is not caused by an ineffective $L_2$ regularization signal. Instead, the difference lies in the induced movement pattern. As shown in Fig.~\ref{fig:weight_change_heatmaps}, masked $L_2$ causes dense movement across many top movement parameters, especially in later stages, whereas masked $L_1$ keeps most top movement parameters nearly unchanged and yields a much sparser movement pattern. In terms of continual learning performance, replacing masked $L_1$ with masked $L_2$ decreases MFN and MTA by 9.96\% and 3.03\%, respectively.

\begin{figure}
    \centering
    \includegraphics[width=0.99\linewidth]{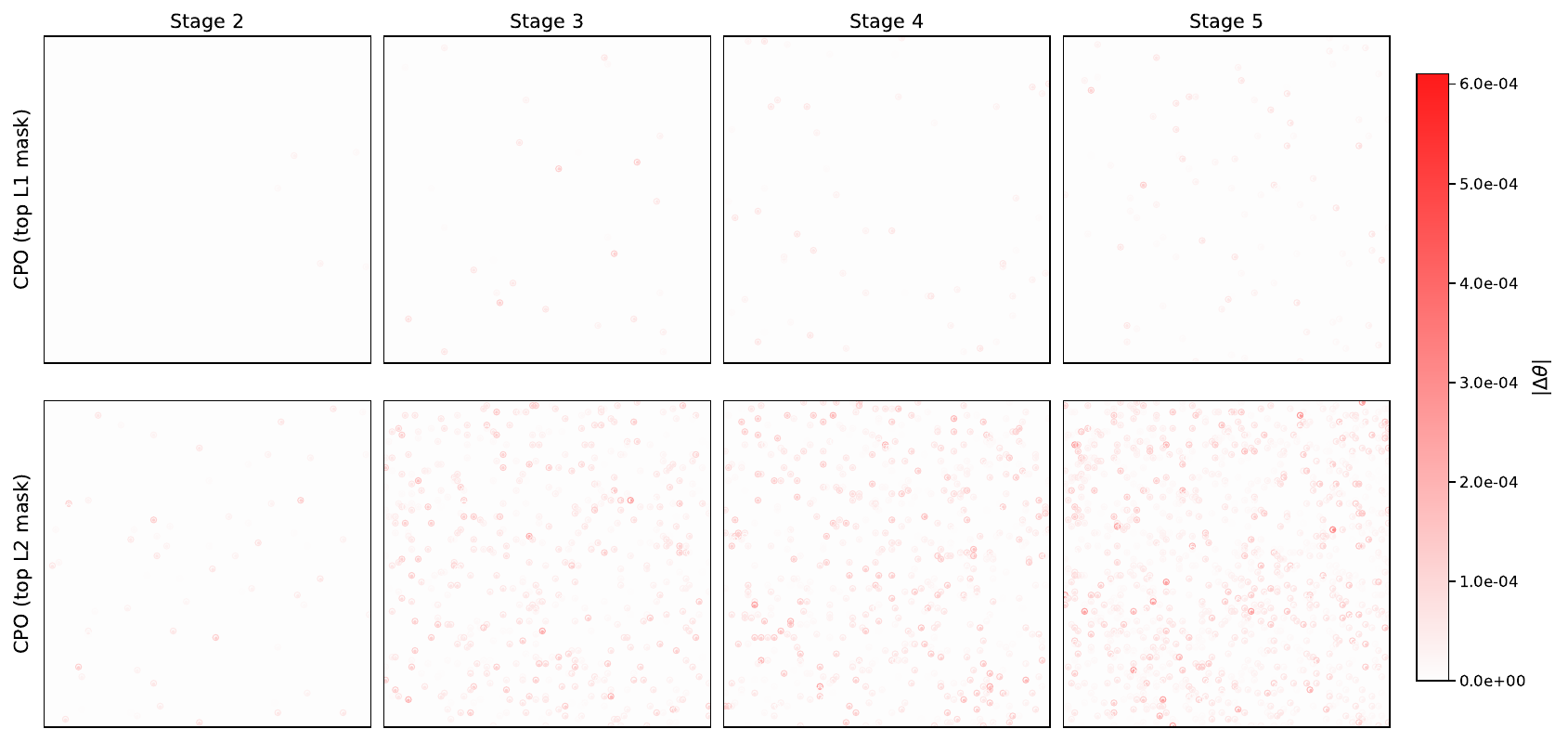}
    \caption{Parameter movement heatmaps of CPO and its masked $L_2$ variant on Qwen3-VL-4B. Masked $L_2$ induces dense movement across many top movement parameters, especially in later continual stages, while masked $L_1$ keeps most top movement parameters close to the previous checkpoint and produces a sparse movement pattern.}
    \label{fig:weight_change_heatmaps}
\end{figure}

\begin{figure}
    \centering
    \includegraphics[width=0.99\linewidth]{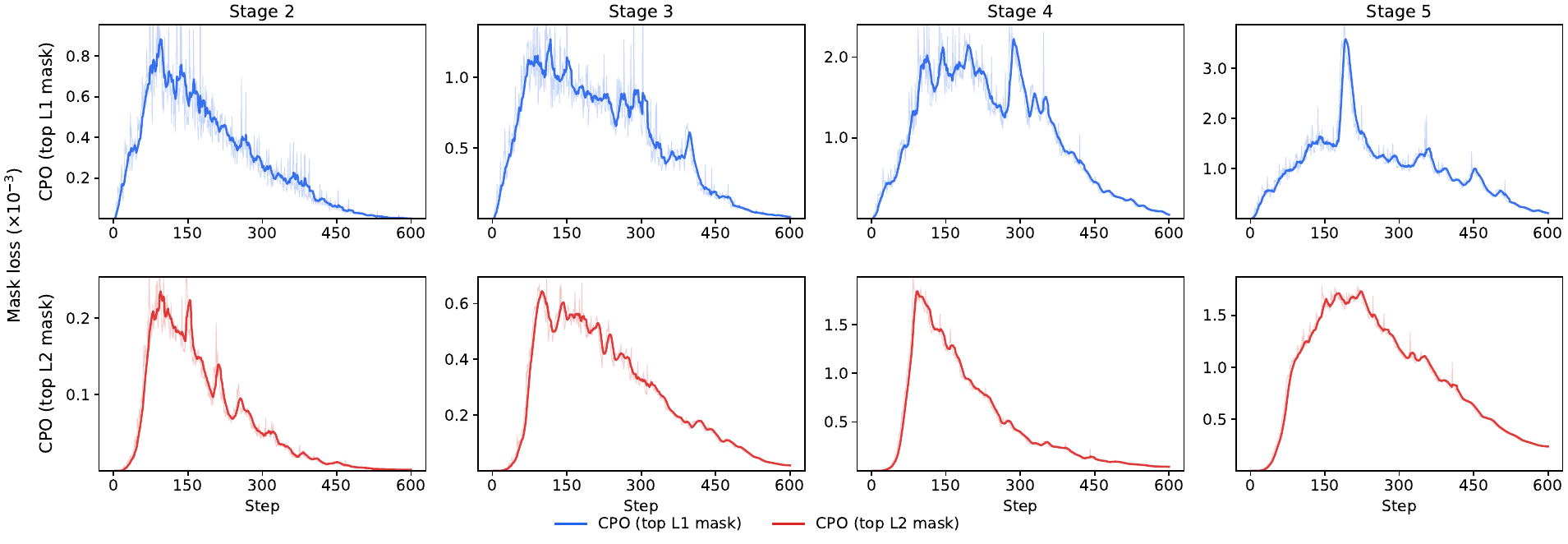}
    \caption{Regularization loss curves of CPO with masked $L_1$ movement regularization and its masked $L_2$ variant on Qwen3-VL-4B. The $L_1$ coefficient is set to 1e2 and the $L_2$ coefficient is set to 1e5. Both regularization losses remain active during training, indicating that the inferior performance of the $L_2$ variant is not due to an inactive regularization signal.}
    \label{fig:coefficientL1L2}
\end{figure}

\section{Stage-by-Stage MRCL Results}
\label{sec:result_details}

This section provides the complete stage-by-stage evaluation results for all compared methods on the MRCL benchmark. Tab.~\ref{tab:qwen2B-CL}, \ref{tab:qwen4B-CL}, and \ref{tab:qwen8B-CL} present the detailed accuracy matrices across the sequential post-training stages for Qwen3-VL-2B, Qwen3-VL-4B, and Qwen3-VL-8B respectively. In each matrix, each row represents the model performance evaluated on a series of tasks after completing the training on the task of that specific stage, illustrating the exact trajectory of learning new capabilities along with the retention of learned knowledge.

\begin{table*}[htbp]
  \centering
  \caption{Accuracy of a series of fine-tuning methods on the MRCL benchmark, evaluated on the Qwen3-VL-2B model. Each row represents the performance trained after the corresponding task.}
  \label{tab:qwen2B-CL}
  \captionsetup[subtable]{justification=centering} 
  
  \vspace{0.8\baselineskip}
  
  \begin{subtable}[htbp]{0.49\textwidth}
    \centering
    \caption{SFT}
    \label{tab:2b-fft-cl}
    \begin{tabular}{b{4.3em}b{1.6em}b{1.6em}b{1.6em}b{1.6em}b{1.6em}}
    \toprule
                 & \rotatebox{60}{MedBook} & \rotatebox{60}{Navigation} & \rotatebox{60}{We-Math2} & \rotatebox{60}{Puzzle} & \rotatebox{60}{FinMME} \\
    \midrule
     MedBook     & \makebox[1.6em]{67.74}  \\
     Navigation  & \makebox[1.6em]{66.21} & \makebox[1.6em]{63.79} \\
     We-Math2     & \makebox[1.6em]{63.16} & \makebox[1.6em]{45.77} & \makebox[1.6em]{38.95} \\
     Puzzle      & \makebox[1.6em]{35.82} & \makebox[1.6em]{2.39} & \makebox[1.6em]{27.67} & \makebox[1.6em]{86.50}  \\
     FinMME      & \makebox[1.6em]{35.65} & \makebox[1.6em]{2.76} & \makebox[1.6em]{29.07} & \makebox[1.6em]{78.50} & \makebox[1.6em]{30.65} \\
    \bottomrule
    \end{tabular}
  \end{subtable}\hfill
  \begin{subtable}[htbp]{0.49\textwidth}
    \centering
    \caption{GRPO*}
    \label{tab:2b-GRPOb0.04-cl}
    \begin{tabular}{b{4.3em}b{1.6em}b{1.6em}b{1.6em}b{1.6em}b{1.6em}}
    \toprule
                 & \rotatebox{60}{MedBook} & \rotatebox{60}{Navigation} & \rotatebox{60}{We-Math2} & \rotatebox{60}{Puzzle} & \rotatebox{60}{FinMME} \\
    \midrule
     MedBook     & \makebox[1.6em]{58.74}  \\
     Navigation  & \makebox[1.6em]{57.22} & \makebox[1.6em]{25.92} \\
     We-Math2     & \makebox[1.6em]{57.89} & \makebox[1.6em]{7.90} & \makebox[1.6em]{48.46} \\
     Puzzle      & \makebox[1.6em]{56.88} & \makebox[1.6em]{7.54} & \makebox[1.6em]{47.82} & \makebox[1.6em]{69.50}  \\
     FinMME      & \makebox[1.6em]{54.16} & \makebox[1.6em]{4.41} & \makebox[1.6em]{46.56} & \makebox[1.6em]{64.99} & \makebox[1.6em]{41.50} \\
    \bottomrule
    \end{tabular}
  \end{subtable}
  
  \vspace{0.8\baselineskip}

  \begin{subtable}[htbp]{0.49\textwidth}
    \centering
    \caption{GRPO}
    \label{tab:2b-GRPOb0-cl}
    \begin{tabular}{b{4.3em}b{1.6em}b{1.6em}b{1.6em}b{1.6em}b{1.6em}}
    \toprule
                 & \rotatebox{60}{MedBook} & \rotatebox{60}{Navigation} & \rotatebox{60}{We-Math2} & \rotatebox{60}{Puzzle} & \rotatebox{60}{FinMME} \\
    \midrule
     MedBook     & \makebox[1.6em]{63.84}  \\
     Navigation  & \makebox[1.6em]{59.08} & \makebox[1.6em]{54.78} \\
     We-Math2     & \makebox[1.6em]{60.44} & \makebox[1.6em]{27.76} & \makebox[1.6em]{54.65} \\
     Puzzle      & \makebox[1.6em]{57.89} & \makebox[1.6em]{27.39} & \makebox[1.6em]{51.56} & \makebox[1.6em]{64.75}  \\
     FinMME      & \makebox[1.6em]{52.29} & \makebox[1.6em]{17.10} & \makebox[1.6em]{50.32} & \makebox[1.6em]{60.25} & \makebox[1.6em]{41.50} \\
    \bottomrule
    \end{tabular}
    
  \end{subtable}\hfill
  \begin{subtable}[htbp]{0.49\textwidth}
    \centering
    \caption{GSPO}
    \label{tab:2b-GSPO-cl}
    \begin{tabular}{b{4.3em}b{1.6em}b{1.6em}b{1.6em}b{1.6em}b{1.6em}}
    \toprule
                 & \rotatebox{60}{MedBook} & \rotatebox{60}{Navigation} & \rotatebox{60}{We-Math2} & \rotatebox{60}{Puzzle} & \rotatebox{60}{FinMME} \\
    \midrule
     MedBook     & \makebox[1.6em]{68.25}  \\
     Navigation  & \makebox[1.6em]{66.04} & \makebox[1.6em]{68.01} \\
     We-Math2     & \makebox[1.6em]{64.18} & \makebox[1.6em]{47.79} & \makebox[1.6em]{54.49} \\
     Puzzle      & \makebox[1.6em]{60.44} & \makebox[1.6em]{25.55} & \makebox[1.6em]{44.89} & \makebox[1.6em]{90.50}  \\
     FinMME      & \makebox[1.6em]{55.52} & \makebox[1.6em]{20.59} & \makebox[1.6em]{47.07} & \makebox[1.6em]{84.00} & \makebox[1.6em]{42.25} \\
    \bottomrule
    \end{tabular}
  \end{subtable}
  \vspace{0.8\baselineskip}
  
  \begin{subtable}[htbp]{0.49\textwidth}
    \centering
    \caption{CPO}
    \label{tab:2b-ours-cl}
    \begin{tabular}{b{4.3em}b{1.6em}b{1.6em}b{1.6em}b{1.6em}b{1.6em}}
    \toprule
                 & \rotatebox{60}{MedBook} & \rotatebox{60}{Navigation} & \rotatebox{60}{We-Math2} & \rotatebox{60}{Puzzle} & \rotatebox{60}{FinMME} \\
    \midrule
     MedBook     & \makebox[1.6em]{68.25}  \\
     Navigation  & \makebox[1.6em]{66.04} & \makebox[1.6em]{68.93} \\
     We-Math2     & \makebox[1.6em]{64.52} & \makebox[1.6em]{50.92} & \makebox[1.6em]{56.10} \\
     Puzzle      & \makebox[1.6em]{61.80} & \makebox[1.6em]{28.31} & \makebox[1.6em]{53.83} & \makebox[1.6em]{88.75}  \\
     FinMME      & \makebox[1.6em]{59.76} & \makebox[1.6em]{23.71} & \makebox[1.6em]{53.71} & \makebox[1.6em]{87.25} & \makebox[1.6em]{44.65} \\
    \bottomrule
    \end{tabular}
  \end{subtable}
\end{table*}

\begin{table*}[htbp]
  \centering
  \caption{Accuracy of a series of fine-tuning methods on the MRCL benchmark, evaluated on the Qwen3-VL-4B model. Each row represents the performance trained after the corresponding task.}
  \label{tab:qwen4B-CL}
  \captionsetup[subtable]{justification=centering} 

  \vspace{0.8\baselineskip}

  \begin{subtable}[htbp]{0.49\textwidth}
    \centering
    \caption{SFT}
    \label{tab:4b-fft-cl}
    \begin{tabular}{b{4.3em}b{1.6em}b{1.6em}b{1.6em}b{1.6em}b{1.6em}}
    \toprule
                 & \rotatebox{60}{MedBook} & \rotatebox{60}{Navigation} & \rotatebox{60}{We-Math2} & \rotatebox{60}{Puzzle} & \rotatebox{60}{FinMME} \\
    \midrule
     MedBook     & \makebox[1.6em]{74.19}  \\
     Navigation  & \makebox[1.6em]{72.50} & \makebox[1.6em]{72.79} \\
     We-Math2     & \makebox[1.6em]{69.10} & \makebox[1.6em]{52.76} & \makebox[1.6em]{46.80} \\
     Puzzle      & \makebox[1.6em]{42.95} & \makebox[1.6em]{14.89} & \makebox[1.6em]{33.16} & \makebox[1.6em]{85.50}  \\
     FinMME      & \makebox[1.6em]{52.12} & \makebox[1.6em]{5.70} & \makebox[1.6em]{37.37} & \makebox[1.6em]{71.75} & \makebox[1.6em]{34.55} \\
    \bottomrule
    \end{tabular}
  \end{subtable}\hfill
  \begin{subtable}[htbp]{0.49\textwidth}
    \centering
    \caption{GRPO*}
    \label{tab:4b-GRPOb0.04-cl}
    \begin{tabular}{b{4.3em}b{1.6em}b{1.6em}b{1.6em}b{1.6em}b{1.6em}}
    \toprule
                 & \rotatebox{60}{MedBook} & \rotatebox{60}{Navigation} & \rotatebox{60}{We-Math2} & \rotatebox{60}{Puzzle} & \rotatebox{60}{FinMME} \\
    \midrule
     MedBook     & \makebox[1.6em]{61.46}  \\
     Navigation  & \makebox[1.6em]{62.31} & \makebox[1.6em]{53.86} \\
     We-Math2     & \makebox[1.6em]{60.27} & \makebox[1.6em]{46.51} & \makebox[1.6em]{63.45} \\
     Puzzle      & \makebox[1.6em]{59.42} & \makebox[1.6em]{30.88} & \makebox[1.6em]{62.92} & \makebox[1.6em]{67.25}  \\
     FinMME      & \makebox[1.6em]{58.23} & \makebox[1.6em]{32.72} & \makebox[1.6em]{62.20} & \makebox[1.6em]{60.25} & \makebox[1.6em]{47.35} \\
    \bottomrule
    \end{tabular}
  \end{subtable}

  \vspace{0.8\baselineskip}

  \begin{subtable}[htbp]{0.49\textwidth}
    \centering
    \caption{GRPO}
    \label{tab:4b-GRPOb0-cl}
    \begin{tabular}{b{4.3em}b{1.6em}b{1.6em}b{1.6em}b{1.6em}b{1.6em}}
    \toprule
                 & \rotatebox{60}{MedBook} & \rotatebox{60}{Navigation} & \rotatebox{60}{We-Math2} & \rotatebox{60}{Puzzle} & \rotatebox{60}{FinMME} \\
    \midrule
     MedBook     & \makebox[1.6em]{70.97}  \\
     Navigation  & \makebox[1.6em]{65.20} & \makebox[1.6em]{46.14} \\
     We-Math2     & \makebox[1.6em]{64.69} & \makebox[1.6em]{11.03} & \makebox[1.6em]{61.39} \\
     Puzzle      & \makebox[1.6em]{63.16} & \makebox[1.6em]{11.03} & \makebox[1.6em]{60.21} & \makebox[1.6em]{74.50}  \\
     FinMME      & \makebox[1.6em]{60.61} & \makebox[1.6em]{15.26} & \makebox[1.6em]{56.21} & \makebox[1.6em]{66.50} & \makebox[1.6em]{43.70} \\
    \bottomrule
    \end{tabular}

  \end{subtable}\hfill
  \begin{subtable}[htbp]{0.49\textwidth}
    \centering
    \caption{Dr.GRPO}
    \label{tab:4b-DrGRPO-cl}
    \begin{tabular}{b{4.3em}b{1.6em}b{1.6em}b{1.6em}b{1.6em}b{1.6em}}
    \toprule
                 & \rotatebox{60}{MedBook} & \rotatebox{60}{Navigation} & \rotatebox{60}{We-Math2} & \rotatebox{60}{Puzzle} & \rotatebox{60}{FinMME} \\
    \midrule
     MedBook     & \makebox[1.6em]{65.20}  \\
     Navigation  & \makebox[1.6em]{62.99} & \makebox[1.6em]{34.74} \\
     We-Math2     & \makebox[1.6em]{60.44} & \makebox[1.6em]{25.37} & \makebox[1.6em]{61.65} \\
     Puzzle      & \makebox[1.6em]{57.22} & \makebox[1.6em]{22.24} & \makebox[1.6em]{58.18} & \makebox[1.6em]{72.00}  \\
     FinMME      & \makebox[1.6em]{54.84} & \makebox[1.6em]{11.03} & \makebox[1.6em]{56.03} & \makebox[1.6em]{58.75} & \makebox[1.6em]{42.85} \\
    \bottomrule
    \end{tabular}
  \end{subtable}

  \vspace{0.8\baselineskip}

  \begin{subtable}[htbp]{0.49\textwidth}
    \centering
    \caption{GSPO}
    \label{tab:4b-GSPO-cl}
    \begin{tabular}{b{4.3em}b{1.6em}b{1.6em}b{1.6em}b{1.6em}b{1.6em}}
    \toprule
                 & \rotatebox{60}{MedBook} & \rotatebox{60}{Navigation} & \rotatebox{60}{We-Math2} & \rotatebox{60}{Puzzle} & \rotatebox{60}{FinMME} \\
    \midrule
     MedBook     & \makebox[1.6em]{69.27}  \\
     Navigation  & \makebox[1.6em]{67.74} & \makebox[1.6em]{85.85} \\
     We-Math2     & \makebox[1.6em]{67.06} & \makebox[1.6em]{59.56} & \makebox[1.6em]{61.49} \\
     Puzzle      & \makebox[1.6em]{61.80} & \makebox[1.6em]{39.89} & \makebox[1.6em]{56.51} & \makebox[1.6em]{87.75}  \\
     FinMME      & \makebox[1.6em]{58.40} & \makebox[1.6em]{43.20} & \makebox[1.6em]{48.82} & \makebox[1.6em]{83.25} & \makebox[1.6em]{46.00} \\
    \bottomrule
    \end{tabular}
  \end{subtable}\hfill
  \begin{subtable}[htbp]{0.49\textwidth}
    \centering
    \caption{CPO}
    \label{tab:4b-ours-cl}
    \begin{tabular}{b{4.3em}b{1.6em}b{1.6em}b{1.6em}b{1.6em}b{1.6em}}
    \toprule
                 & \rotatebox{60}{MedBook} & \rotatebox{60}{Navigation} & \rotatebox{60}{We-Math2} & \rotatebox{60}{Puzzle} & \rotatebox{60}{FinMME} \\
    \midrule
     MedBook     & \makebox[1.6em]{69.27}  \\
     Navigation  & \makebox[1.6em]{69.78} & \makebox[1.6em]{83.64} \\
     We-Math2     & \makebox[1.6em]{68.59} & \makebox[1.6em]{72.24} & \makebox[1.6em]{66.20} \\
     Puzzle      & \makebox[1.6em]{67.91} & \makebox[1.6em]{70.59} & \makebox[1.6em]{64.91} & \makebox[1.6em]{79.75}  \\
     FinMME      & \makebox[1.6em]{68.93} & \makebox[1.6em]{73.71} & \makebox[1.6em]{65.66} & \makebox[1.6em]{78.50} & \makebox[1.6em]{49.80} \\
    \bottomrule
    \end{tabular}
  \end{subtable}
\end{table*}

\begin{table*}[htbp]
  \centering
  \caption{Accuracy of a series of fine-tuning methods on the MRCL benchmark, evaluated on the Qwen3-VL-8B model. Each row represents the performance trained after the corresponding task.}
  \label{tab:qwen8B-CL}
  \captionsetup[subtable]{justification=centering} 

  \vspace{0.8\baselineskip}

  \begin{subtable}[htbp]{0.49\textwidth}
    \centering
    \caption{SFT}
    \label{tab:8b-fft-cl}
    \begin{tabular}{b{4.3em}b{1.6em}b{1.6em}b{1.6em}b{1.6em}b{1.6em}}
    \toprule
                 & \rotatebox{60}{MedBook} & \rotatebox{60}{Navigation} & \rotatebox{60}{We-Math2} & \rotatebox{60}{Puzzle} & \rotatebox{60}{FinMME} \\
    \midrule
     MedBook     & \makebox[1.6em]{79.46}  \\
     Navigation  & \makebox[1.6em]{79.46} & \makebox[1.6em]{70.04} \\
     We-Math2    & \makebox[1.6em]{74.70} & \makebox[1.6em]{59.01} & \makebox[1.6em]{50.89} \\
     Puzzle      & \makebox[1.6em]{69.61} & \makebox[1.6em]{36.03} & \makebox[1.6em]{48.12} & \makebox[1.6em]{74.50}  \\
     FinMME      & \makebox[1.6em]{62.31} & \makebox[1.6em]{4.60} & \makebox[1.6em]{43.48} & \makebox[1.6em]{64.50} & \makebox[1.6em]{38.05} \\
    \bottomrule
    \end{tabular}
  \end{subtable}\hfill
  \begin{subtable}[htbp]{0.49\textwidth}
    \centering
    \caption{GRPO}
    \label{tab:8b-GRPOb0-cl}
    \begin{tabular}{b{4.3em}b{1.6em}b{1.6em}b{1.6em}b{1.6em}b{1.6em}}
    \toprule
                 & \rotatebox{60}{MedBook} & \rotatebox{60}{Navigation} & \rotatebox{60}{We-Math2} & \rotatebox{60}{Puzzle} & \rotatebox{60}{FinMME} \\
    \midrule
     MedBook     & \makebox[1.6em]{71.65}  \\
     Navigation  & \makebox[1.6em]{69.78} & \makebox[1.6em]{73.53} \\
     We-Math2    & \makebox[1.6em]{68.59} & \makebox[1.6em]{50.92} & \makebox[1.6em]{65.30} \\
     Puzzle      & \makebox[1.6em]{65.53} & \makebox[1.6em]{43.20} & \makebox[1.6em]{63.81} & \makebox[1.6em]{72.25}  \\
     FinMME      & \makebox[1.6em]{61.29} & \makebox[1.6em]{42.83} & \makebox[1.6em]{59.87} & \makebox[1.6em]{71.25} & \makebox[1.6em]{47.55} \\
    \bottomrule
    \end{tabular}
  \end{subtable}

  \vspace{0.8\baselineskip}

  \begin{subtable}[htbp]{0.49\textwidth}
    \centering
    \caption{GSPO}
    \label{tab:8b-GSPO-cl}
    \begin{tabular}{b{4.3em}b{1.6em}b{1.6em}b{1.6em}b{1.6em}b{1.6em}}
    \toprule
                 & \rotatebox{60}{MedBook} & \rotatebox{60}{Navigation} & \rotatebox{60}{We-Math2} & \rotatebox{60}{Puzzle} & \rotatebox{60}{FinMME} \\
    \midrule
     MedBook     & \makebox[1.6em]{72.33}  \\
     Navigation  & \makebox[1.6em]{69.44} & \makebox[1.6em]{79.23} \\
     We-Math2    & \makebox[1.6em]{68.25} & \makebox[1.6em]{43.20} & \makebox[1.6em]{64.85} \\
     Puzzle      & \makebox[1.6em]{67.91} & \makebox[1.6em]{36.21} & \makebox[1.6em]{60.79} & \makebox[1.6em]{85.00}  \\
     FinMME      & \makebox[1.6em]{61.12} & \makebox[1.6em]{40.07} & \makebox[1.6em]{58.96} & \makebox[1.6em]{72.25} & \makebox[1.6em]{45.40} \\
    \bottomrule
    \end{tabular}
  \end{subtable}\hfill
  \begin{subtable}[htbp]{0.49\textwidth}
    \centering
    \caption{CPO}
    \label{tab:8b-ours-cl}
    \begin{tabular}{b{4.3em}b{1.6em}b{1.6em}b{1.6em}b{1.6em}b{1.6em}}
    \toprule
                 & \rotatebox{60}{MedBook} & \rotatebox{60}{Navigation} & \rotatebox{60}{We-Math2} & \rotatebox{60}{Puzzle} & \rotatebox{60}{FinMME} \\
    \midrule
     MedBook     & \makebox[1.6em]{72.33}  \\
     Navigation  & \makebox[1.6em]{72.50} & \makebox[1.6em]{83.27} \\
     We-Math2    & \makebox[1.6em]{72.33} & \makebox[1.6em]{82.54} & \makebox[1.6em]{65.23} \\
     Puzzle      & \makebox[1.6em]{71.82} & \makebox[1.6em]{77.94} & \makebox[1.6em]{64.02} & \makebox[1.6em]{76.50}  \\
     FinMME      & \makebox[1.6em]{71.82} & \makebox[1.6em]{79.60} & \makebox[1.6em]{64.45} & \makebox[1.6em]{72.75} & \makebox[1.6em]{48.80} \\
    \bottomrule
    \end{tabular}
  \end{subtable}

  \vspace{0.8\baselineskip}
\end{table*}

\clearpage
\section{Prompt Templates}
\label{sec:prompts}
In the experiments, two prompt strategies are used. Under the direct prompts setting, the model is restricted to output only the final answer. The CoT prompts elicit explicit step-by-step reasoning from the model before yielding the prediction, where the final answer is enclosed within boxed\{\}.
Concretely, we apply direct prompts to Zero-shot$^1$ and SFT, while CoT prompts are used for Zero-shot$^2$ and all RL methods, including GRPO, GRPO$^*$, Dr.GRPO, GSPO, and CPO.
\begin{promptbox}[Direct Prompts: ]{lightgreen}
\textbf{MedBookVQA:}\\
      "You are an AI assistant helping to solve the following problem based on the image(s). 
      
      Instructions: 
      
      Please output the answer directly and NO extra thinking. 
      
      Rules for the final answer: 
      
      - Multiple Choice: Output the uppercase option letter ONLY (e.g., A, B, C, D, E...).
      
      \{question\}"
\\\\
\textbf{Navigation:}\\
"You are an AI assistant helping to solve the following problem based on the image(s). 

Instructions: 

Please output the answer directly and NO extra thinking. 

Note: There is one and only one shortest path. 

Rules for the final answer: 

- Only provide your solution as a sequence of capital letter which should be comma-separated. 

\{question\}"
\\\\
\textbf{We-Math2:}\\
      "You are an AI assistant helping to solve the following problem based on the image(s). 
      
      Instructions: 
      
      Please output the answer directly and NO extra thinking.
      
      Rules for the final answer: 
      
      - Multiple Choice: Output the uppercase option letter ONLY (e.g., A, B, C, D, E...). 
      
      - Yes/No: Output Yes or No ONLY. 
      
      - Numerical: Output the number ONLY. 
      
      - Fill-in-the-blank: Output the answer phrase or word directly. 
      
      - Math Formula: Output the LaTeX expression enclosed in single dollar signs (e.g., $x=y$). 
      
      \{question\}"
\\\\
\textbf{Puzzle:}\\
"You are an AI assistant helping to solve the following problem based on the image(s). 

Instructions: 

Please output the answer directly and NO extra thinking. 

Rules for the final answer: 

- Multiple Choice: Output the uppercase option letter ONLY (e.g., A, B, C, D, E...).

\{question\}"
\\\\
\textbf{FinMME:}\\
"You are an AI assistant helping to solve the following problem based on the image(s). 

Instructions: 

Please output the answer directly and NO extra thinking. 

Rules for the final answer: 

- Multiple Choice: If one or more options are correct, output the uppercase option letter(s) ONLY. If there are multiple answers, separate them with commas (e.g., A,C,D). 

- Numerical: Output the number ONLY.

\{question\}"
\end{promptbox}
\clearpage
\begin{promptbox}[CoT Prompts: ]{lightblue}
\textbf{MedBookVQA:}\\
      "You are an AI assistant helping to solve the following problem based on the image(s). 
      
      Instructions: 
      
      Please think step-by-step and put the final answer in \verb|\\boxed{{}}|.
      
      Rules for the final answer in \verb|\\boxed{{}}|: 
      
      - Multiple Choice: Output the uppercase option letter ONLY (e.g., A, B, C, D, E...).
      
      \{question\}"
\\\\
\textbf{Navigation:}\\
"You are an AI assistant helping to solve the following problem based on the image(s). 

Instructions: 

Please think step-by-step and put the final answer in \verb|\\boxed{{}}|.

Hints: 

You need to first determine the size of the grid, then determine the coordinates of the hole, player, and gift, and finally perform path planning.

Note: There is one and only one shortest path. 

Rules for the final answer in \verb|\\boxed{{}}|:

- Only provide your solution as a sequence of capital letter which should be comma-separated. 

\{question\}"
\\\\
\textbf{We-Math2:}\\
      "You are an AI assistant helping to solve the following problem based on the image(s). 
      
      Instructions: 
      
      Please think step-by-step and put the final answer in \verb|\\boxed{{}}|.
      
      Rules for the final answer in \verb|\\boxed{{}}|: 
      
      - Multiple Choice: Output the uppercase option letter ONLY (e.g., A, B, C, D, E...). 
      
      - Yes/No: Output Yes or No ONLY. 
      
      - Numerical: Output the number ONLY. 
      
      - Fill-in-the-blank: Output the answer phrase or word directly. 
      
      - Math Formula: Output the LaTeX expression enclosed in single dollar signs (e.g., $x=y$). 
      
      \{question\}"
\\\\
\textbf{Puzzle:}\\
"You are an AI assistant helping to solve the following problem based on the image(s). 

Instructions: 

Please think step-by-step and put the final answer in \verb|\\boxed{{}}|.

Rules for the final answer in \verb|\\boxed{{}}|: 

- Multiple Choice: Output the uppercase option letter ONLY (e.g., A, B, C, D, E...).

\{question\}"
\\\\
\textbf{FinMME:}\\
"You are an AI assistant helping to solve the following problem based on the image(s). 

Instructions: 

Please think step-by-step and put the final answer in \verb|\\boxed{{}}|.

Rules for the final answer in \verb|\\boxed{{}}|: 

- Multiple Choice: If one or more options are correct, output the uppercase option letter(s) ONLY. If there are multiple answers, separate them with commas (e.g., A,C,D). 

- Numerical: Output the number ONLY.

\{question\}"
\end{promptbox}

\section{Case Study}
\label{sec:Case Study}

This section presents a qualitative case study on Qwen3-VL-4B across diverse datasets to illustrate the evolution of model behavior during continual post-training. We compare model responses to identical test samples when evaluated immediately after learning each specific task versus after completing the entire sequential training process. The results show that GSPO initially solves these examples correctly but suffers from forgetting after completing the final FinMME stage. Its failures involve visual reconstruction, mathematical reasoning, navigation, and medical image interpretation. In contrast, CPO consistently preserves both the step-by-step reasoning structure and the visual comprehension acquired during earlier stages without degrading its logical chain.

\begin{outerbox}
    \begin{innerblock}{questionbg}
        {\large\bfseries\color{titlebrown}TASK: Puzzle} \\[0.3em]
        \InsertBoxR{0}{\includegraphics[width=0.35\textwidth]{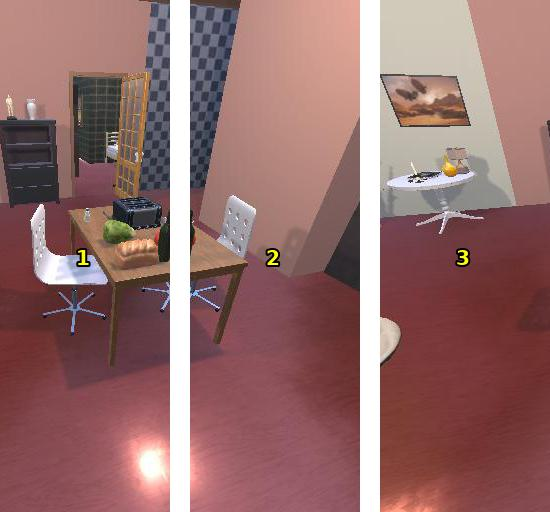}}
        \textbf{Question :} \\
        The image below is divided into three parts by white strips, forming a $1\times3$ jigsaw puzzle. The parts are labeled “1”, “2”, and “3” from left to right. These parts are from a single original image but have been shuffled. Your task is to determine the correct arrangement of the physically labeled parts to reconstruct the natural image.\\[0.6em]
        A. The left part should be Part 2; the middle part should be Part 1; and the right part should be Part 3.\\
        B. The left part should be Part 2; the middle part should be Part 3; and the right part should be Part 1.\\
        C. The left part should be Part 1; the middle part should be Part 3; and the right part should be Part 2.\\
        D. The left part should be Part 3; the middle part should be Part 1; and the right part should be Part 2.

    \end{innerblock}
    \begin{innerblock}{gspobg1}
        \textbf{GSPO (train after Puzzle):} \\[0.3em]
        The image shows three components: - The original image. - The three parts. - The three images. The goal is to match the labels. The correct sequence is: 1. 3 2. 1 3. 2. Thus, the correct answer is:$\backslash$$\backslash$boxed\{D\} \quad \textcolor{green!70!black}{\ding{52}}
        
    \end{innerblock}

    \begin{innerblock}{gspobg2}
        \textbf{GSPO (train after FinMME):} \\[0.3em]
        The three parts show different perspectives: - Left shows two tables. - Middle shows one table. - Right shows no table. The correct sequence is: 1. 2. 3. 4. Hence, the best answer is:$\backslash$$\backslash$boxed\{C\} (based on logical reasoning) \quad \textcolor{red}{\ding{55}}

    \end{innerblock}
    \begin{innerblock}{oursbg}
        \textbf{CPO (train after FinMME):} \\[0.3em]
        Let's analyze the scene: We are looking at a room with furniture and decor. The key elements are: - Part 1: A dining table with chairs, food items, and a black box. There’s a chair on the left and another on the right. - Part 2: An empty space next to the table, showing the wall and floor. It appears to be the “empty” area between the table and the wall. - Part 3: A small round table with a painting on the wall above it. This looks like a more open or living area. Now, let’s consider how these would fit together logically: The most logical layout would be: → Start with the wall and painting (Part 3) on the left — this gives us the background context. → Then the dining area (Part 1) — which is the main room activity. → Then the empty space (Part 2) — which is likely the continuation of the room’s layout, perhaps leading to a doorway or hallway. So the correct order is: Part 3 → Part 1 → Part 2. That corresponds to: Left: Part 3. Middle: Part 1. Right: Part 2. This matches: D. The left part should be Part 3; the middle part should be Part 1; and the right part should be Part 2. This makes sense visually and spatially.$\backslash$$\backslash$boxed\{D\}. \quad \textcolor{green!70!black}{\ding{52}}

    \end{innerblock}
\end{outerbox}

\clearpage
\begin{outerbox}
    \begin{innerblock}{questionbg}
        {\large\bfseries\color{titlebrown}TASK: We-Math2.0} \\[0.3em]
        \InsertBoxR{0}{\includegraphics[width=0.25\textwidth]{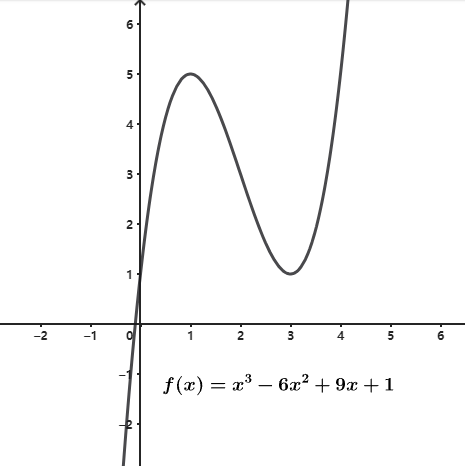}}
        \textbf{Question :} \\
        A toy manufacturing company produces a new type of building block toy. According to market research, the relationship between the daily production quantity of building blocks $x$ (in thousands of sets) and the profit per set $y$ (in yuan) can be represented by the function shown in the figure. The company's production planning department wants to understand the fluctuation in profit. If $x$ is restricted to the interval $[0,3]$, what is the value of the maximum of this function minus its minimum?

    \end{innerblock}
    \begin{innerblock}{gspobg1}
        \textbf{GSPO (train after We-Math2):} \\[0.3em]
       The given function is: $f(x)=x^3-6x^2+9x+1$. We are to find the maximum and minimum values of this function on the interval $[0,3]$, and then compute the difference: max - min. Step 1: Find critical points. To find the extrema, we take the derivative: $f'(x)=3x^2-12x+9$. Set $f'(x)=0$: $3x^2-12x+9=0 \quad x^2-4x+3=0 \quad (x-1)(x-3)=0$. So, critical points at $x=1$ and $x=3$. Note: Both are within $[0,3]$. Step 2: Evaluate the function at critical points and endpoints. Endpoints: $x=0$, $x=3$. Critical points: $x=1$, $x=3$ (already included). Let's compute: 1. At $x=0$: $f(0)=0-0+0+1=1$. 2. At $x=1$: $f(1)=1-6+9+1=5$. 3. At $x=3$: $f(3)=27-54+27+1=1$. Wait --- let's check the calculation at $x=3$: Actually, let's compute carefully: $f(3)=3^3-6*(9)+9*3+1=27-54+27+1=1$. And at $x=1$: $1-6+9+1=5$. At $x=0$: $1$. Therefore, the function reaches: - Maximum: 5 (at $x=1$). - Minimum: 1 (at both $x=0$ and $x=3$). Thus, the difference: $5-1=4$. Final Answer:$\backslash$$\backslash$boxed\{4\}. \quad \textcolor{green!70!black}{\ding{52}}
   
    \end{innerblock}

    \begin{innerblock}{gspobg2}
        \textbf{GSPO (train after FinMME):} \\[0.3em]
        To solve the problem: From the graph: - Maximum occurs approximately at $(1,5)$. - Minimum occurs approximately at $(3,-1)$. Therefore, the difference would be: $(\text{Maximum}-\text{Minimum})=5-(-1)=6$. Therefore, the final answer should be:$\backslash$$\backslash$boxed\{6\}. \quad \textcolor{red}{\ding{55}}
        
    \end{innerblock}
    \begin{innerblock}{oursbg}
        \textbf{CPO (train after FinMME):} \\[0.3em]
        We are given the function: $f(x)=x^3-6x^2+9x+1$ and we are to find the difference between the maximum and minimum values of this function on the interval $[0,3]$. Step 1: Find critical points by taking derivative. Compute $f'(x)$: $f'(x)=3x^2-12x+9$. Set $f'(x)=0$: $3x^2-12x+9=0 \quad x^2-4x+3=0 \quad (x-1)(x-3)=0 \Rightarrow x=1 \text{ or } x=3$. So, critical points at $x=1$ and $x=3$. Also, endpoints of the interval are $x=0$ and $x=3$. Thus, we need to evaluate $f(x)$ at: $x=0$, $x=1$, $x=3$. Step 2: Compute function values. At $x=0$: $f(0)=0^3-6(0)^2+9(0)+1=1$. At $x=1$: $f(1)=1-6+9+1=5$. At $x=3$: $f(3)=27-54+27+1=(27-54)+(27+1)=(-27)+28=1$. So, the values are:- $f(0)=1$ - $f(1)=5$ - $f(3)=1$. Step 3: Determine max and min on $[0,3]$. From above: - Maximum = 5 (at x=1) - Minimum = 1 (at x=0 and x=3). So, the difference: $\text{Max}-\text{Min}=5-1=4$.\\
        $\backslash$$\backslash$boxed\{4\}. \quad \textcolor{green!70!black}{\ding{52}}
    \end{innerblock}
\end{outerbox}

\clearpage
\begin{outerbox}
    \begin{innerblock}{questionbg}
        {\large\bfseries\color{titlebrown}TASK: Navigation} \\[0.3em]
        \InsertBoxR{0}{\includegraphics[width=0.24\textwidth]{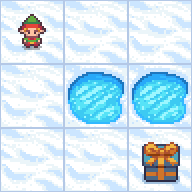}}
        \textbf{Question :} \\
        You are a maze solver. Your goal is to guide a player from the start to the goal on a grid map while avoiding holes.\\
        The player can move one square at a time in the directions left (L), right (R), up (U), or down (D). \\
        The frozen lake is not slippery; the player will always move in the intended direction.\\
        Moving off the edge or falling into a hole results in failure. Reaching the goal means success.\\
        Provide your solution as a sequence of moves, such as L,R,U,D. The moves should be comma-separated.

    \end{innerblock}
    \begin{innerblock}{gspobg1}
        \textbf{GSPO (train after Navigation):} \\[0.3em]
        Let's analyze the grid step by step. The grid appears to be a 3x3 (3 rows and 3 columns). Label coordinates (row, column), starting from $(0,0)$ in the top-left. From the image: - The player is at $(0,0)$. - The holes are at $(1,1)$ and $(1,2)$. - The gift is at $(2,2)$. We need to plan the shortest path from $(0,0)\to(2,2)$ without stepping on a hole. Current path options: 1. Move down (D) to $(1,0)$ — that square is clear. 2. Move right (R) to $(0,1)$ — that square is clear (ice). Let's plan the shortest route: From $(0,0)$: 1. Move down (D) to $(1,0)$ — clear. 2. Move down (D) to $(2,0)$ — clear. 3. Move right (R) to $(2,1)$ — clear. 4. Move right (R) to $(2,2)$ — the gift. That would be: D, D, R, R — 4 moves. Alternatively, from $(0,0)\to(0,1)\to(1,1)$ — but $(1,1)$ is a hole. So we cannot go there. So, the correct path is: $\backslash$$\backslash$boxed\{D, D, R, R\} \quad \textcolor{green!70!black}{\ding{52}}
        
    \end{innerblock}

    \begin{innerblock}{gspobg2}
        \textbf{GSPO (train after FinMME):} \\[0.3em]
        To solve this: From the image: - Player position: $(0,0)$. - Goal position: $(2,2)$. - Holes positions: $(1,1)$ and $(1,2)$. Therefore, the path should be: - From $(0,0)\to(0,1)$: Move Right (R) to $(0,1)$ — but that's not a hole. - Move Down (D) to $(1,0)$ — that's not a hole. - Move Down (D) to $(2,0)$ — that's not a hole. - Move Right (R) to $(2,1)$ — that's not a hole. - Move Down (D) — already at the goal. Therefore, the actual path would be: - R, D, D, R. Therefore, the final answer is:$\backslash$$\backslash$boxed\{R,D,D,R\} \quad \textcolor{red}{\ding{55}}
    \end{innerblock}
    \begin{innerblock}{oursbg}
        \textbf{CPO (train after FinMME):} \\[0.3em]
        Let's analyze the grid: The grid has 3 rows and 3 columns. Row 0: [Player] [ ] [ ]. Row 1: [ ] [Hole] [Hole]. Row 2: [ ] [ ] [Gift]. So positions are: $(0,0)$ = Player, $(1,1)$ = Hole, $(1,2)$ = Hole, $(2,2)$ = Gift. We need to get from $(0,0)$ to $(2,2)$ without stepping on holes. Possible moves: From $(0,0)$: - Down to $(1,0)$ — safe (no hole). - Right to $(0,1)$ — safe. From $(1,0)$: - Down to $(2,0)$ — safe. - Right to $(1,1)$ — hole $\to$ invalid. From $(2,0)$: - Right to $(2,1)$ — safe. - Up to $(1,0)$ — already visited. From $(2,1)$: - Right to $(2,2)$ — GOAL! So the path is: $(0,0)\to(1,0)\to(2,0)\to(2,1)\to(2,2)$. That's: Down (D) $\to$ Down (D) $\to$ Right (R) $\to$ Right (R). But let's check if there's a shorter path. Alternatively: $(0,0)\to$ Right (R) $\to(0,1)\to$ Down (D) $\to(1,1)$ — hole $\to$ invalid. So the only valid path is: $(0,0)\to D\to(1,0)\to D\to(2,0)\to R\to(2,1)\to R\to(2,2)$. That's 7 moves? But we can optimize: Actually, from $(1,0)\to(2,0)\to(2,1)\to(2,2)$ is 3 moves. So total moves: Start at $(0,0)$ $\to$ D to $(1,0)$ $\to$ D to $(2,0)$ $\to$ R to $(2,1)$ $\to$ R to $(2,2)$. So moves: D, D, R, R. That's 4 moves. Is that correct? Yes. So the answer is:$\backslash$$\backslash$boxed\{D, D, R, R\} \quad \textcolor{green!70!black}{\ding{52}}
    \end{innerblock}
\end{outerbox}

\clearpage
\begin{outerbox}
    \begin{innerblock}{questionbg}
        {\large\bfseries\color{titlebrown}TASK: MedBookVQA} \\[0.3em]
        \InsertBoxR{0}{\includegraphics[width=0.25\textwidth]{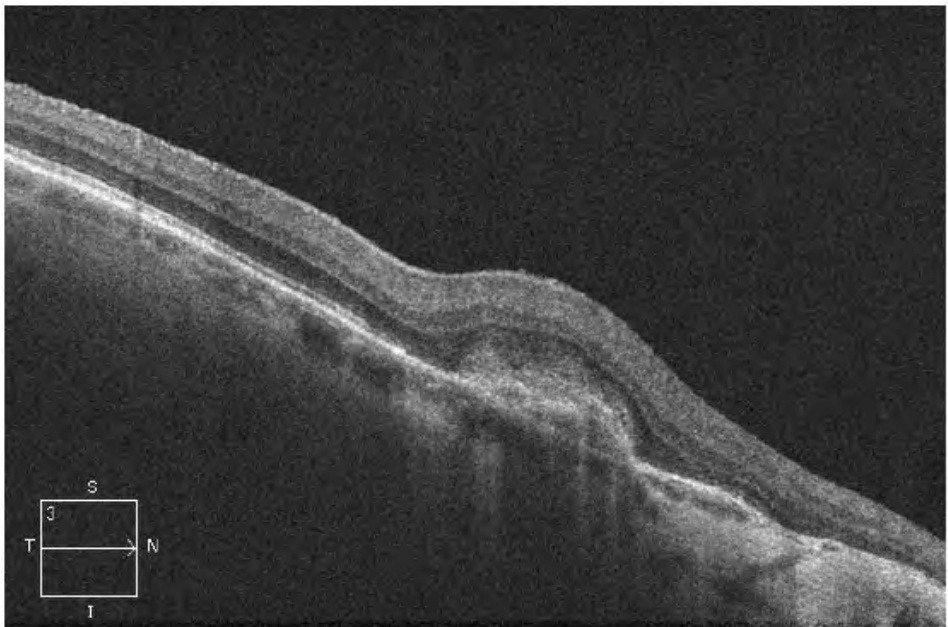}}
        \textbf{Question :} \\
        What abnormality is visible in the subretinal space? \\
        A. Hyporeflective fluid\\
        B. Isolated dark spots\\
        C. Hyperreflective material\\
        D. Uniformly reflective layer

    \end{innerblock}
    \begin{innerblock}{gspobg1}
        \textbf{GSPO (train after MedBookVQA):} \\[0.3em]
        Let's analyze the image step-by-step: 1. The image is an OCT (optical coherence tomography) scan, which is used to visualize the layers of the retina and subretinal space. 2. In the image, we can see a layered structure with a clear separation between the retinal layers and the subretinal space (the space beneath the retina). 3. The key observation is the presence of a bright, reflective line or band that appears to be ``floating'' or ``separated'' from the retinal layers — this is often associated with hyperreflective material or a lesion. 4. In contrast, hyporeflective fluid would appear as a dark, low-reflectivity space, which is not what we see here. 5. The image does not show isolated dark spots or a uniformly reflective layer — it shows a more complex, structured pattern. 6. The most plausible interpretation is that there is hyperreflective material present in the subretinal space, which can be seen as a bright, reflective band or structure that stands out against the darker background. Therefore, the correct answer is:$\backslash$$\backslash$boxed\{C\}. \quad \textcolor{green!70!black}{\ding{52}}
        
    \end{innerblock}

    \begin{innerblock}{gspobg2}
        \textbf{GSPO (train after FinMME):} \\[0.3em]
        To solve this: 1. Observe the key features: In the image, note the greyish, wave-like structure above the black background. 2. Look for abnormalities: The grey structure shows a kind of ``bump'' or ``hump'' which is often associated with pathology. 3. Identify the location: The hump appears to be located below the grey structure, which would be the subretinal space. Therefore, the correct interpretation should be: - What is visible? The grey structure shows a kind of ``hump'' or ``bump''. - What is the nature? This is typically seen in conditions such as retinal detachment, subretinal fluid, etc. Therefore, the correct answer should match the instructions: - Final Answer should be boxed. Therefore, the final answer is:$\backslash$$\backslash$boxed\{A\} Because the grey structure shows a kind of ``hump'' which can be interpreted as hyporeflective fluid or other abnormalities. However, from the options, the best choice would be the one that matches the actual observation. Actually, from the image, we can observe that the grey structure shows a kind of ``hump'' which is often associated with hyporeflective fluid or other abnormalities. Therefore, the final answer is:$\backslash$$\backslash$boxed\{A\}. \quad \textcolor{red}{\ding{55}}
        
    \end{innerblock}
    \begin{innerblock}{oursbg}
        \textbf{CPO (train after FinMME):} \\[0.3em]
        Let's analyze the image and the options step-by-step: 1. The image is an OCT (optical coherence tomography) scan, which is used to visualize the internal structure of the eye. 2. In OCT, different tissues and structures appear as varying degrees of brightness or darkness: - Hyporeflective = darker, less reflective (often fluid or spaces). - Hyperreflective = brighter, more reflective (often tissue or material with high backscatter). 3. The key feature in this image is the ``subretinal space'' — the space between the retina and the retinal pigment epithelium (RPE). 4. In a normal eye, the subretinal space is typically very thin or absent (or just a small gap). However, in this case, there appears to be a space that is not just empty but has some kind of material or fluid. 5. Looking at the image, we can see a bright line (likely representing the RPE or retinal layers) and below it, a space that appears to have a ``dark'' or ``hyporeflective'' appearance — which is consistent with fluid or space. 6. But note: the key clue is the presence of a ``bright'' or ``hyperreflective'' material or structure within the subretinal space — which is not typical of fluid (fluid would be hyporeflective). Instead, this suggests a hyperreflective material — possibly a lesion, debris, or even a foreign body. 7. Therefore, the correct interpretation is: $\rightarrow$ The abnormality is hyperreflective material in the subretinal space. This is consistent with conditions like subretinal hemorrhage, subretinal deposits, or even certain types of retinal detachments where material is present. Thus, the correct answer is:$\backslash$$\backslash$boxed\{C\} \quad \textcolor{green!70!black}{\ding{52}}
    \end{innerblock}
\end{outerbox}

\end{document}